\documentclass[10pt,twocolumn,letterpaper]{article}

\usepackage{cvpr}
\usepackage{times}
\usepackage{epsfig}
\usepackage{graphicx}
\usepackage{amsmath}
\usepackage{amssymb}
\usepackage{algorithm}
\usepackage{algorithmic}
\usepackage{threeparttable}
\usepackage{setspace}
\usepackage{multirow}
\usepackage{wrapfig}
\usepackage{subcaption}

\usepackage[utf8]{inputenc}
\usepackage[english]{babel}
\usepackage{amsthm}

\newtheorem{proposition}{Proposition}

\DeclareMathOperator*{\argmin}{argmin}

\usepackage[pagebackref=true,breaklinks=true,letterpaper=true,colorlinks,bookmarks=false]{hyperref}

\cvprfinalcopy 

\def\cvprPaperID{3487} 

\ifcvprfinal\fi
\begin{document}

\title{Deeply-supervised Knowledge Synergy}

\author{Dawei Sun$^{1,2*}$ ~~ Anbang Yao$^{1*}$ ~~ Aojun Zhou$^1$ ~~ Hao Zhao$^{1,2}$\\
\and
$^1$Intel Labs China\quad $^2$Tsinghua University\\
\tt\small \{dawei.sun, anbang.yao, aojun.zhou, hao.zhao\}@intel.com\\
}

\maketitle
\thispagestyle{empty}

\begin{abstract}
\let\thefootnote\relax\footnotetext{*Equal contribution. This work was done when Dawei Sun was an intern at Intel Labs China, supervised by Anbang Yao who is responsible for correspondence. Interns Aojun Zhou and Hao Zhao contributed to early theoretical analysis.}
Convolutional Neural Networks (CNNs) have become deeper and more complicated compared with the pioneering AlexNet. However, current prevailing training scheme follows the previous way of adding supervision to the last layer of the network only and propagating error information up layer-by-layer. In this paper, we propose Deeply-supervised Knowledge Synergy (DKS), a new method aiming to train CNNs with improved generalization ability for image classification tasks without introducing extra computational cost during inference. Inspired by the deeply-supervised learning scheme, we first append auxiliary supervision branches on top of certain intermediate network layers. While properly using auxiliary supervision can improve model accuracy to some degree, we go one step further to explore the possibility of utilizing the probabilistic knowledge dynamically learnt by the classifiers connected to the backbone network as a new regularization to improve the training. A novel synergy loss, which considers pairwise knowledge matching among all supervision branches, is presented. Intriguingly, it enables dense pairwise knowledge matching operations in both top-down and bottom-up directions at each training iteration, resembling a dynamic synergy process for the same task. We evaluate DKS on image classification datasets using state-of-the-art CNN architectures, and show that the models trained with it are consistently better than the corresponding counterparts. For instance, on the ImageNet classification benchmark, our ResNet-152 model outperforms the baseline model with a $1.47\%$ margin in Top-1 accuracy. Code is available at \url{https://github.com/sundw2014/DKS}.

\end{abstract}

\section{Introduction}
Deep Convolutional Neural Networks (CNNs) have large numbers of learnable parameters, which makes them have much better capability in fitting training data than traditional machine learning methods. Along with the growing availability of training resources including large-scale datasets, powerful hardware platforms and effective development tools, CNNs have become the dominant learning models for a variety of visual recognition tasks~\cite{ref01,ref02,ref03,ref04}. In order to get more compelling performance, CNNs~\cite{ref14,ref15,ref16,ref17,ref18,ref19,ref20} are designed to be considerably deeper and more complicated in comparison to the seminal AlexNet~\cite{ref01} which has 8 layers and achieved groundbreaking results in the ImageNet classification competition 2012~\cite{ref21}. Despite that modern CNNs widely use various engineering techniques such as careful hyper-parameter tuning~\cite{ref14}, aggressive data argumentation~\cite{ref18,ref13}, effective normalization~\cite{ref07,ref08} and sophisticated connection path~\cite{ref15,ref17,ref18,ref19,ref20} to ease network training, their training remains to be difficult.

\begin{figure}
\centering
\includegraphics[width = 0.95\linewidth]{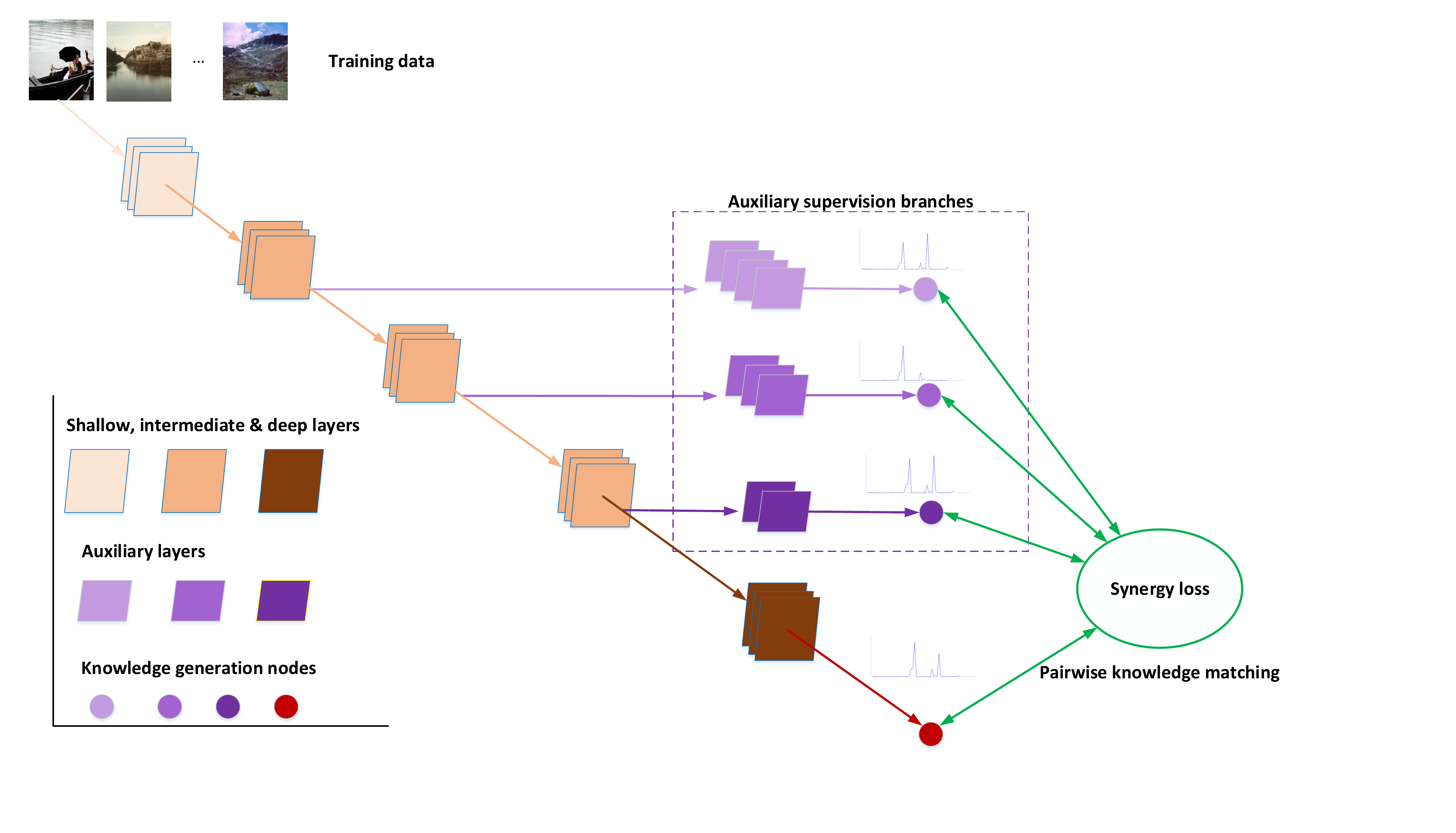}
\caption{Illustration of the proposed method. In the figure, we add three auxiliary supervision branches on top of some intermediate layers of the backbone network. Every branch will output a class probability distribution conditioned on the training data, which is used as the knowledge. We use circles to indicate the nodes for calculating these knowledge outputs, and propose a synergy loss term to enable the pair-wise matching among them. Best viewed electronically.}
\label{fig:01}
\vskip -0.1 in
\end{figure}

We notice that state-of-the-art CNN models such as ResNet~\cite{ref15}, WRN~\cite{ref16}, DenseNet~\cite{ref17}, ResNeXt~\cite{ref18}, SENet~\cite{ref19}, DPN~\cite{ref20}, MobileNet~\cite{ref27,ref29} and ShuffleNet~\cite{ref28,ref28-2} adopt the training scheme of AlexNet. More specifically, during training, the supervision is only added to the last layer of the network and the training error is back propagated from the last layer to earlier layers. Because of the increased complexity in network depth, building blocks and network topologies, this might pose a risk of insufficient representation learning, especially to the layers from which there are long connection paths to the supervision layer. This problem may be alleviated by the deeply-supervised learning scheme proposed in~\cite{ref12} and~\cite{ref22} independently. Szegedy et al.~\cite{ref12} add auxiliary classifiers to two intermediate layers of their proposed GoogLeNet, while Lee et al.~\cite{ref22} propose to add auxiliary classifiers to all hidden layers of the network. During network training, although different types of auxiliary classifiers are used in these two methods, they adopt the same optimization strategy in which the training loss is the weighted sum of the losses of all auxiliary classifiers and the loss of the classifier connected to the last layer. Such methodology has proven to be notably effective in combating the vanishing gradient problem and overcoming the convergence issue for training some old deep classification networks. However, modern CNN backbones usually have no convergence issue, and rarely use auxiliary classifiers. Recently, Huang et al.~\cite{ref23} present a two-dimensional multi-scale CNN architecture using early-exit classifiers for cost-aware image classification. In~\cite{ref23}, empirical results show that naively attaching simple auxiliary classifiers to the early layers of a state-of-the-art CNN such as ResNet or DenseNet leads to decreased performance, but this issue can be alleviated with a combination of multi-scale features and dense connections from the architecture design perspective.

In this paper, we revisit the deeply-supervised learning methodology for image classification tasks, and present a new method called Deeply-supervised Knowledge Synergy (DKS) targeting to train state-of-the-art CNNs with improved accuracy and without introducing extra computational cost during inference. Inspired by the aforementioned works ~\cite{ref12,ref22,ref23}, we first append auxiliary supervision branches on top of certain intermediate layers during network training as illustrated in Fig.~\ref{fig:01}. We show that using carefully designed auxiliary classifiers can improve the accuracy of state-of-the-art CNNs to a certain extent. This empirically indicates that the information from the auxiliary supervision is beneficial in regularizing the training of modern CNNs. We conjecture there may still exist room for performance improvement by enabling explicit information interactions among all supervision branches connected to the backbone network, thus we go one step further to explore the possibility of utilizing the knowledge (namely the class probability outputs evaluated on the training data) dynamically learnt by the auxiliary classifiers and the classifier added to the last network layer as a new regularization to improve the training. In the optimization, a novel synergy loss, which considers pairwise knowledge matching among all supervision branches, is added to the training loss. This loss enables dense pairwise knowledge matching operations in both top-down and bottom-up directions at each training step, resembling a dynamic synergy process for the same task. We evaluate the proposed method on two well-known image classification datasets using the most prevalent CNN architectures including ResNet~\cite{ref15}, WRN~\cite{ref16}, DenseNet~\cite{ref17} and MobileNet~\cite{ref27} architectures. We show that the models trained with our method have impressive accuracy improvements compared with their respective baseline models. For example, on the challenging ImageNet classification dataset, even to very deep ResNet-152 architecture, there is a $1.47\%$ improvement in Top-1 accuracy.

\section{Related Work}

Here, we summarize related approaches in the literature, and analyze their relations and differences with our method.


\textbf{Deeply-Supervised Learning.} The deeply-supervised learning methodology~\cite{ref12,ref22} was released in 2014. It uses auxiliary classifiers connected to the hidden layers of the network to address the convergence problem when training some old deep CNNs for image classification tasks. Recently, it has also been used in other visual recognition tasks such as edge detection~\cite{ref32}, human pose estimation~\cite{ref33}, scene parsing~\cite{ref34}, semantic segmentation~\cite{ref34-2}, keypoint localization~\cite{ref35}, automatic delineation~\cite{ref48} and travel time estimation~\cite{ref36}. Despite these recent advances in its new applications, modern CNN classification models rarely use auxiliary classifiers. As reported in~\cite{ref23}, directly appending simple auxiliary classifiers on top of the early layers of a state-of-the-art network such as ResNet or DenseNet hurts its performance. In this paper, we present DKS, a new deeply-supervised learning method for image classification tasks, which shows impressive accuracy improvements when training state-of-the-art CNNs.

\textbf{Knowledge Transfer.} In the recent years, Knowledge Transfer (KT) research has been attracting increasing interest. A pioneering work is Knowledge Distillation (KD)~\cite{ref37} in which the soft outputs from a large teacher model or an ensemble of teacher models are used to regularize the training of a smaller student network.~\cite{ref38},~\cite{ref39} and~\cite{ref37-2} further show that intermediate feature representations can also be used as hints to enhance knowledge distillation process. KD techniques have also been used in other tasks, for instance, improving the performance of low-precision CNNs for image classification~\cite{ref40} and designing multiple-stream CNNs for video action recognition~\cite{ref37-5}. Unlike KD and its variants in which knowledge is only transferred from teacher models to a student model,~\cite{ref41} extends KD by presenting a mutual learning strategy, showing that the knowledge of the student model is also helpful to improve the accuracy of the teacher model. Later, this idea was used in face re-identification~\cite{ref41-3} and joint human parsing and pose estimation~\cite{ref41-2}. Li and Hoiem~\cite{ref42} address the problem of adapting a trained neural network model to handle new vision tasks while preserving the old knowledge through a combination of KD and fine-tuning. An improved method is proposed in~\cite{ref42-2}. Qiao et al.~\cite{ref37-3} propose a deep co-training method for semi-supervised image classification. In their method, all models are considered as students and trained with different data views containing adversarial samples. In this paper, the proposed deeply-supervised knowledge synergy method is a new form of knowledge transfer within one single neural network, which differs from the aforementioned methods both in focus and formulation.

\textbf{CNN Regularization.} ReLU~\cite{ref05}, Dropout~\cite{ref06} and BN~\cite{ref07} are proven to be the keys for modern CNNs to combat over-fitting or accelerate convergence. Because of this, many improved variants~\cite{ref08,ref09,ref10,ref06-2,ref06-3} have been proposed recently. Over-fitting can also be reduced by synthetically increasing the size of existing training data via augment transformations such as random cropping, flipping, scaling, color manipulation and linear interpolation~\cite{ref01,ref11,ref12,ref13}. In addition, pre-training~\cite{ref14} can assist the early stages of the neural network training. These methods are widely used in modern CNN architecture design and training. Our method is compatible with them. As can be seen in Fig.~\ref{fig:03}, the model trained with DKS has the highest training error but the lowest test error, showing that our method behaves like a regularizer and reduces over-fitting for ResNet-18.

\section{The Proposed Method}

In this section, we present the formulation of our method, highlight its insight, and detail its implementation.

\subsection{Deeply-Supervised Learning}

We begin with the formulation of the deeply-supervised learning scheme as our method is based on it. Let $ W_{c} $ be the parameters of a $L$-layer CNN model that needs to be learnt. Let $D = \{(x_i, y_i) | 1\leq i \leq N\}$ be an annotated data set having $N$ training samples collected from $K$ image classes. Here, $x_i$ is the $i^{th}$ training sample and $y_i$ is the corresponding ground truth label (a one-hot vector with $K$ dimensions). Let $f(W_{c}, x_i)$ be the $K$-dimensional output vector of the CNN model for a training sample $x_i$. For the standard training scheme, the supervision is only added to the last layer of the network, and the optimization objective can be defined as
\begin{equation}\label{eq:01}
\begin{aligned}
\argmin_{W_{c}} \quad& {L_{c}}(W_{c}, D) + \lambda{R(W_{c})}, \\
\end{aligned}
\end{equation}
where $L_{c}$ is the default loss, $R$ is the regularization term, and $\lambda$ is a positive coefficient. Here, $L_{c}$ is defined as
\begin{equation*}\label{eq:02}
\begin{aligned}
L_{c}(W_{c}, D) = \frac{1}{N} \sum\limits_{i=1}^NH(y_i,f(W_{c}, x_i)), \\
\end{aligned}
\end{equation*}
where $H$ is a cross-entropy cost function
\begin{equation*}\label{eq:03}
\begin{aligned}
H(y_i,f(W_{c}, x_i)) = -\sum\limits_{k=1}^Ky_i^k\log f^{k}(W_{c}, x_i). \\
\end{aligned}
\end{equation*}
As $\lambda R$ is a default term and has no relation with our method, we omit this term in the following description for simplicity. Now, the objective function~(\ref{eq:01}) can be reduced into
\begin{equation}\label{eq:04}
\begin{aligned}
\argmin_{W_{c}} \quad& {L_{c}}(W_{c}, D). \\
\end{aligned}
\end{equation}
This optimization problem can be readily solved by SGD and its variants~\cite{ref44,ref45,ref46}. To the best of our knowledge, most of the well-known CNNs~\cite{ref01,ref14,ref15,ref16,ref17,ref18,ref27,ref29,ref19,ref20,ref28,ref28-2,ref30,ref31,ref31-2} adopt this optimization scheme in the model training. By contrast, the deeply-supervised learning scheme explicitly proposed in~\cite{ref22} adds auxiliary classifiers to all hidden layers of the network during training. Let $W_{a} = \{w_{a}^{l} | 1\leq l \leq L-1\}$ be a set of auxiliary classifiers attached on the top of every hidden layer of the network. Here, $w_{a}^{l}$ denotes the parameters of the auxiliary classifier added to the $l^{th}$ hidden layer. Let $f(w_{a}^{l}, W_{c}, x_i)$ be the $K$-dimensional output vector of the $l^{th}$ auxiliary classifier. Without loss of generality, the optimization objective of the deeply-supervised learning scheme can be defined as
\begin{equation}\label{eq:05}
\begin{aligned}
\argmin_{W_{c}, W_{a}} \quad& {L_{c}}(W_{c}, D) + L_{a}(W_{a}, W_{c}, D), \\
\end{aligned}
\end{equation}
where
\begin{equation*}\label{eq:06}
\begin{aligned}
L_{a}(W_{a}, W_{c}, D) = \frac{1}{N} \sum\limits_{i=1}^N\sum\limits_{l=1}^{L-1}\alpha_lH(y_i, f(w_{a}^{l}, W_{c}, x_i)). \\
\end{aligned}
\end{equation*}
The auxiliary loss $L_{a}$ is the weighted sum of the losses of all auxiliary classifiers evaluated on the training set and $\alpha_l$ weights the loss of the $l^{th}$ auxiliary classifier. By introducing auxiliary loss $L_{a}$, the deeply-supervised learning scheme allows the network to gather gradients not only from the last layer supervision but also from the hidden layer supervision during training. This is thought to combat the vanishing gradient problem and enhance convergence~\cite{ref22,ref12}.

As for the contemporary work~\cite{ref12}, its optimization objective can be thought as a special case of~(\ref{eq:05}) as it only adds auxiliary classifiers to two intermediate layers of the proposed GoogLeNet. The other difference lies in the structure of auxiliary classifiers. In the experiments,~\cite{ref22} used simple classifiers with a zero-ing strategy to dynamically control the value of $\alpha_l$ during training, while~\cite{ref12} used more complex classifiers with a fixed value of $\alpha_l$. We find that setting a fixed value for $\alpha_l$ gives similar performance to the zero-ing strategy when training state-of-the-art CNNs, thus we use fixed values for $\alpha_l$ in our implementation.

\subsection{Deeply-supervised Knowledge Synergy}

Now, we present the formulation of our DKS which further develops the deeply-supervised learning methodology from a new perspective. DKS also uses auxiliary classifiers connected to some hidden layers of the network, but unlike existing methods, it introduces explicit information interactions among all supervision branches. Specifically, DKS uses the knowledge (i.e., the class probability outputs evaluated on the training data) dynamically learnt by all classifiers to regularize network training. Its core contribution is a novel synergy loss which enables dense pairwise knowledge matching among all classifiers connected to the backbone network, making optimization more effective.

In this section, we follow the notations in the last section. We only add auxiliary classifiers to certain hidden layers. Let $A \subseteq \{1,2,\cdot\cdot\cdot,L-1\}$ be a pre-defined set with $|A|$ layer indices, indicating where auxiliary classifiers are added. Let $\hat A = A \cup \{L\}$, where $L$ is the index of the last layer of the network, so that $\hat{A}$ indicates the locations of all classifiers connected to the network including both the auxiliary ones and the original one. Let $B \subseteq \hat{A} \times \hat{A}$ be another pre-defined set with $|B|$ pairs of layer indices, indicating where pair-wise knowledge matching operations are activated.

Now, following the definition of~(\ref{eq:05}), the optimization objective of our DKS is defined as
\begin{equation}\label{eq:09}
\resizebox{.9\hsize}{!}{$
\begin{aligned}
\argmin_{W_{c},W_{a}} \quad& {L_{c}}(W_{c},D)+L_{a}(W_{a},W_{c},D)+L_{s}(W_{a},W_{c},D).\\
\end{aligned}$}
\end{equation}
Here, the default loss $L_{c}$ is the same as in~(\ref{eq:05}), the auxiliary loss $L_{a}$ is defined as
\begin{equation*}\label{eq:10}
\begin{aligned}
L_{a}(W_{a}, W_{c}, D) = \frac{1}{N} \sum\limits_{i=1}^N\sum\limits_{l \in A}\alpha_lH(y_i, f(w_{a}^{l}, W_{c}, x_i)), \\
\end{aligned}
\end{equation*}
and the proposed synergy loss $L_{s}$ is defined as
\begin{equation*}\label{eq:11}
\begin{aligned}
L_{s}(W_{a}, W_{c}, D) = \frac{1}{N} \sum\limits_{i=1}^N\sum\limits_{(m,n) \in B}H(f_m,f_n).\\
\end{aligned}
\end{equation*}
The pairwise knowledge matching from the classifier $m$ to $n$ is evaluated with $H(f_m,f_n)$ which is defined as
\begin{equation*}\label{eq:12}
\begin{aligned}
H(f_m,f_n) = -\beta_{mn}\sum\limits_{k=1}^Kf_m^k\log f_n^k, \\
\end{aligned}
\end{equation*}
where $f_m$ and $f_n$ are the class probability outputs of classifier $m$ and $n$ evaluated on the training sample $x_i$ respectively, and $\beta_{mn}$ weights the loss of the pairwise knowledge matching from the classifier $m$ to $n$. We use a Softmax function to compute class probability. In the experiments, we set $\alpha_l=1$, $\beta_{mn}=1$ and keep them fixed, which means there is no extra hyper-parameter in the optimization of our method compared with the optimization~(\ref{eq:04}) and~(\ref{eq:05}). For the synergy loss, the knowledge matching between any two classifiers is a modified cross-entropy loss function with a soft target. In principle, taking the current class probability outputs from the classifier $m$ as soft labels (which are considered as constant values and the gradients w.r.t. them will not be calculated in the back-propagation), it forces the classifier $n$ to mimic the classifier $m$. In this way, the knowledge currently learnt by the classifier $m$ can be transferred to the classifier $n$. We call this term a directional supervision. Intriguingly, enabling dense pairwise knowledge matching operations among all supervision branches connected to the backbone network resembles a dynamic synergy process for the same task.

\begin{figure}
\centering
\includegraphics[height=3.5cm]{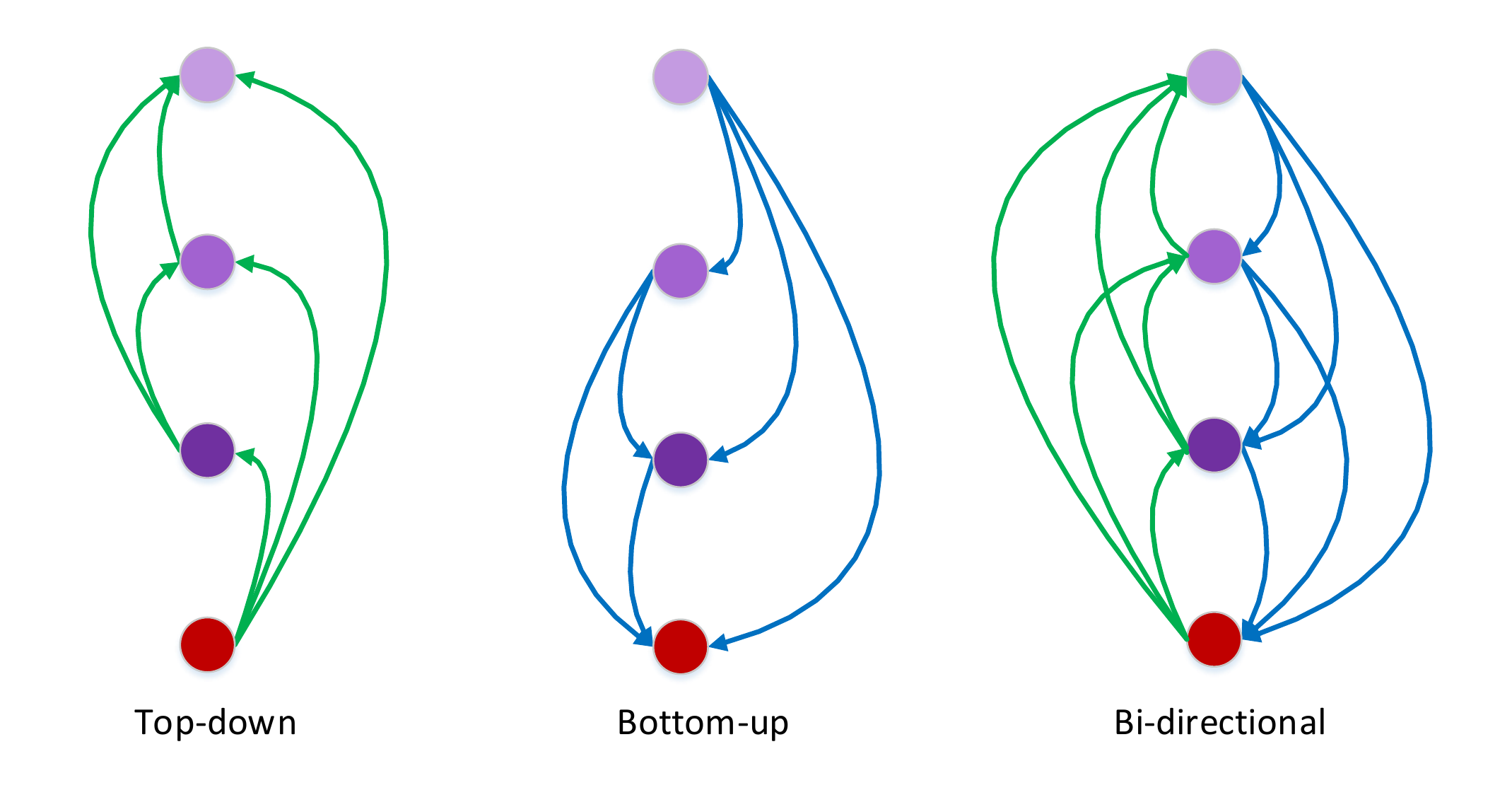}
\caption{Illustration of three pairwise knowledge matching strategies. In each strategy, the red circle denotes the classifier connected to the last layer of the network, and the purple circles denote three auxiliary classifiers connected to certain intermediate layers, and the curved arrows represent the pairwise knowledge matching directions.}
\label{fig:02}
\vskip -0.1 in
\end{figure}

\textbf{Pairwise Knowledge Matching}. For DKS, a critical question is how to configure the knowledge matching pairs (i.e., set $B$). We provide three options including the top-down, bottom-up and bi-directional strategies, as illustrated in Fig.~\ref{fig:02}. With the top-down strategy, only the knowledge of the classifiers connected to the deep layers of a backbone network are used to guide the training of the classifiers added to the earlier layers. The bottom-up strategy reverses this setting and the bi-directional strategy includes both of them. A comparison study (see experiments section) shows that the bi-directional strategy has the best performance, thus we adopt it in the final implementation.

\textbf{Auxiliary Classifiers}. Another basic question for DKS is how to design the structure of auxiliary classifiers. Although the deeply-supervised learning scheme has proven to be effective in addressing the convergence issue when training some old deep networks for image classification tasks~\cite{ref22}, state-of-the-art CNNs such as ResNet and DenseNet are known to be free of convergence issue, even for models having hundreds of layers. In view of this, directly adding simple auxiliary classifiers to the hidden layers of the network might not be helpful, which has been empirically verified by~\cite{ref23} and~\cite{ref34-2}. From the CNN architecture design perspective,~\cite{ref12} and~\cite{ref23} propose to add complex auxiliary classifiers to some intermediate layers of the network to alleviate this problem. Following them, in the experiments, we append relatively complex auxiliary supervision branches on top of certain intermediate layers during network training. Specifically, every auxiliary branch is composed of the same building block (e.g., residual block in ResNet) as in the backbone network. As empirically verified in~\cite{ref23}, early layers lack coarse-level features which are helpful for image-level classification. In order to address this problem, we use a heuristic principle making the paths from the input to every classifier have the same number of down-sampling layers. Comparative experiments show that these carefully designed auxiliary supervision branches can improve final model performance to some extent but the gain is relatively minor. By enabling dense pairwise knowledge matching via the proposed synergy loss, we achieve much better results. Fig.~\ref{fig:03} shows some illustrative results, and more results can be found in experiments section.

\begin{figure}
\centering
\includegraphics[width=.95\linewidth]{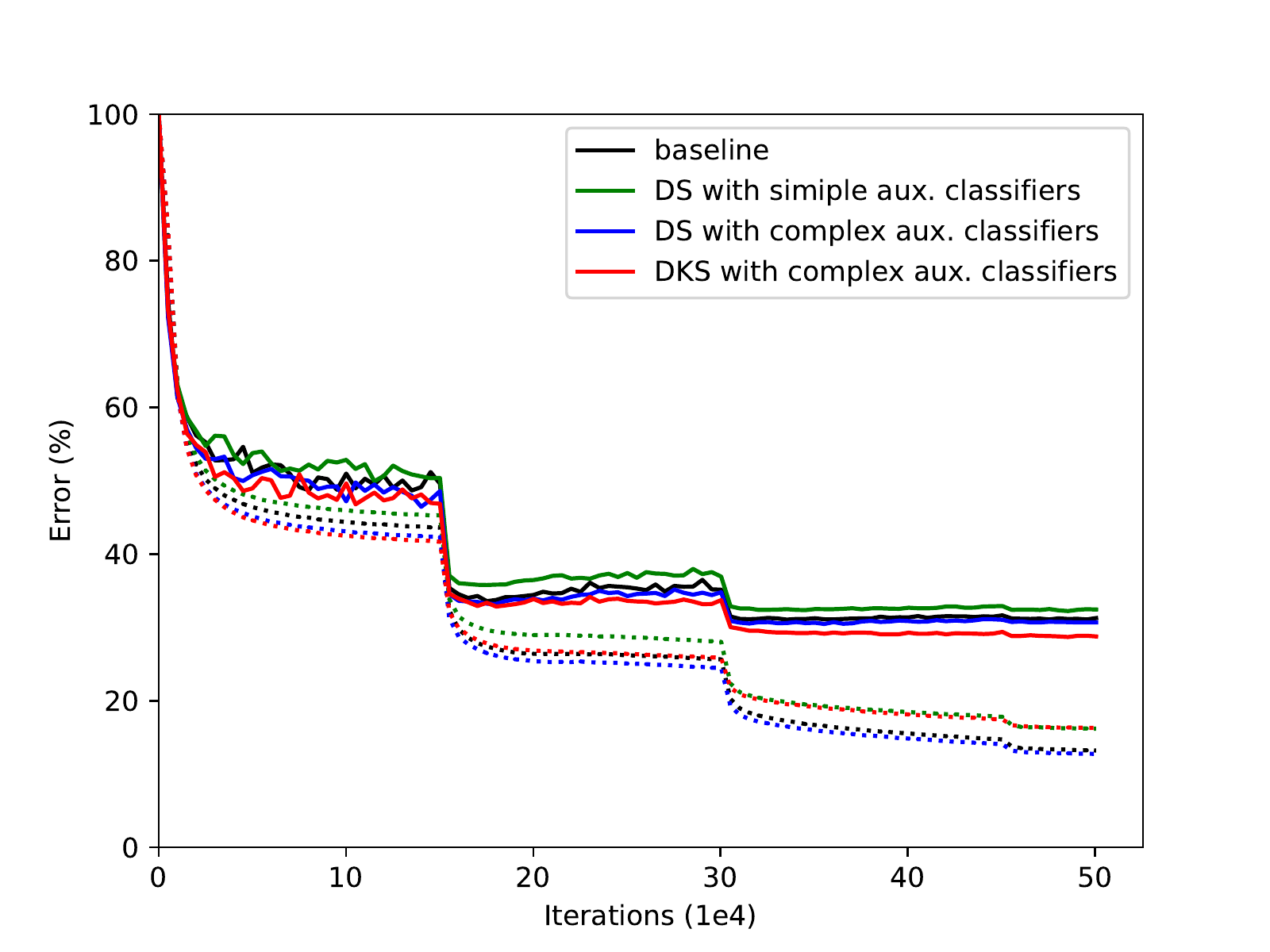}
\caption{Curves of Top-1 training error (dashed line) and test error (solid line) of the ResNet-18 models trained on the ImageNet classification dataset. Compared with the baseline model, simple auxiliary classifiers (added after the block Conv3$\_$x and Conv4$\_$x) lead to $1.17\%$ drop in Top-1 accuracy and complex designs bring a $0.60\%$ improvement, while our method achieves $2.38\%$ gain. Remarkably, our method converges with the lowest accuracy on training set but achieves the best accuracy on test set, showing better capability in suppressing over-fitting.}
\label{fig:03}
\vskip -0.1 in
\end{figure}

\textbf{Comparison with Knowledge Distillation}. In the DKS, the pairwise knowledge matching is inspired by the knowledge distillation idea popularly used in knowledge transfer~\cite{ref37,ref37-2,ref38,ref39,ref40,ref41,ref42,ref42-2,ref37-3}. Here, we clarify their differences. First, our method differs with them in focus. This line of research mainly addresses the network compression problem following a student-teacher framework, but our method focuses on advancing the training of state-of-the-art CNNs by further developing the deeply-supervised learning methodology. Second, our method differs with them in formulation. Under the student-teacher framework, large teacher models are usually supposed to be available beforehand, and the optimization is defined to use the soft outputs from teacher models to guide the training of smaller student networks. That is, teacher models and student models are separately optimized, and there is no direct relation between them. In our method, auxiliary classifiers share different-level feature layers of the backbone network, and they are jointly optimized with the classifier connected to the last layer. In this paper, we also conduct experiments to compare their performance.

To the best of our knowledge, DKS is the first work that makes a compact association of deeply-supervised learning and knowledge distillation methodologies, enabling the transfer of currently learned knowledge between different layers in a deep CNN model. In the supplemental materials, we provide some theoretical analysis attempting to better understand DKS.

\section{Experiments}

In this section, we first apply DKS to train state-of-the-art CNNs on the CIFAR-100~\cite{ref50} and ImageNet~\cite{ref21} classification datasets, and compare it with the standard training scheme and the Deeply-Supervised (DS) learning scheme. We then provide experiments for a deep analysis of DKS and more comprehensive comparisons. All algorithms are implemented with PyTorch~\cite{ref49}. For fair comparisons, the experiments of these three methods are conducted with exactly the same settings for data pre-processing, batch size, number of training epochs, learning rate scheduling, etc.

\subsection{Experiments on CIFAR-100}

CIFAR-100 dataset~\cite{ref50} contains 50000 training images and 10000 test images, where instances are $32\times32$ color images drawn from 100 object classes. We use the same data pre-processing method as in~\cite{ref15,ref22}. For training, images are padded with 4 pixels to both sides first, and then $32\times32$ crops are randomly sampled from the padded images or their horizontal flips, and are finally normalized with the per-channel mean and std values. For evaluation, we report the error on the original-sized test images.

\textbf{Backbone Networks and Implementation Details.} We consider four state-of-the-art CNN architectures including: (1) ResNets~\cite{ref15} with depth 32 and 110; (2) DenseNets~\cite{ref17} with depth 40/100 and growth rate 12; (3) WRNs~\cite{ref16} with depth 28/28 and widening factor 4/10; (4) MobileNet~\cite{ref27} as used in~\cite{ref41}. We use the released code by the authors and follow the standard settings to train each respective backbone. During training, for ResNets and MobileNet, we use SGD with momentum, and we set the batch size as 64, the weight decay as 0.0001, the momentum as 0.9 and the number of training epochs as 200. The initial learning rate is 0.1, and it is divided by 10 every 60 epochs. For DenseNets, we use SGD with Nesterov momentum, and we set the batch size as 64, the weight decay as 0.0001, the momentum as 0.9 and the number of training epochs as 300. The initial learning rate is set to 0.1, and is divided by 10 at 50\% and 75\% of the total number of training epochs. For WRNs, we use SGD with momentum, and we set the batch size as 128, the weight decay as 0.0005, the momentum as 0.9 and the number of training epochs as 200. The initial learning rate is set to 0.1, and is divided by 5 at 60, 120 and 160 epochs. Inspired by~\cite{ref12,ref23}, we append three auxiliary classifiers to certain intermediate layers of these CNN architectures. Specifically, we add each auxiliary classifier after the corresponding building block having a down-sampling layer. All auxiliary classifiers have the same building blocks as in the backbone networks, a global average pooling layer and a fully connected layer. The differences are the number of building blocks and the number of convolutional filters (see supplementary materials for details). All models are trained on a server using 1 GPU. For each network, we run each method 5 times and report `mean(std)' error rates.

\textbf{Results Comparison.} Results are summarized in Table~\ref{1} where baseline denotes the standard training scheme, and DS denotes the deeply-supervised learning scheme~\cite{ref12,ref22} using our designed auxiliary classifiers. Generally, with our designed auxiliary classifiers, DS improves model accuracy in all cases compared to the baseline method, and its accuracy gain ranges from $0.08\%$ to $0.92\%$. Comparatively, our method performs the best on all networks, bringing at least $0.67\%$ and at most $3.08\%$ accuracy gain to DS. As the network goes to much deeper (e.g., ResNet-110 and DenseNet-100)/much wider (e.g., WRN-28-10)/much thinner (e.g., MobileNet), our method also has noticeable accuracy improvements over all counterparts. These experiments clearly validate the effectiveness of the proposed method when training state-of-the-art CNNs.

\begin{table}
\begin{center}
\resizebox{.9\linewidth}{!}{
\begin{tabular}{p{1.8cm}p{1.5cm}p{1.5cm}p{1.8cm}c}
\hline
Model&Method&Error(\%)&Average gain(\%)\\
\hline
 \multirow{3}{*}{\shortstack[l]{ResNet\\(d=32)}}            & baseline & 29.97(0.33) & -    \\
                                                            & DS   & $29.89(0.26)$ & 0.08 \\
                                                            & DKS     & \textbf{26.81(0.36)} & \textbf{3.16} \\ \hline
 \multirow{3}{*}{\shortstack[l]{ResNet\\(d=110)}}           & baseline & 27.66(0.60) & -    \\
                                                            & DS       & 26.95(0.51) & 0.71 \\
                                                            & DKS     & \textbf{24.98(0.35)} & \textbf{2.68} \\ \hline
 \multirow{3}{*}{\shortstack[l]{DenseNet\\(d=40,\,k=12)}}    & baseline & 24.91(0.18) & -    \\
                                                            & DS       & $24.46(0.22)$ & 0.45\\
                                                            & DKS     & \textbf{23.61(0.20)} & \textbf{1.30} \\ \hline
 \multirow{3}{*}{\shortstack[l]{DenseNet\\(d=100,\,k=12)}}  & baseline & 20.92(0.31) & -    \\
                                                            & DS       & 20.34(0.23) & 0.58 \\
                                                            & DKS     & \textbf{19.67(0.29)} & \textbf{1.25} \\ \hline
 \multirow{3}{*}{\shortstack[l]{WRN-28-4}}        & baseline & 21.39(0.30) & -    \\
                                                            & DS       & 20.47(0.21) & 0.92 \\
                                                            & DKS     & \textbf{18.91(0.08)} & \textbf{2.48} \\ \hline
 \multirow{3}{*}{\shortstack[l]{WRN-28-10}}         & baseline & 18.72(0.24) & -    \\
                                                            & DS       & 18.32(0.13) & 0.40 \\
                                                            & DKS     & \textbf{17.24(0.22)} & \textbf{1.48} \\ \hline
 \multirow{3}{*}{\shortstack[l]{WRN-28-10\\(0.3 dropout)}}         & baseline & 18.64(0.19) & -    \\
                                                            & DS       & 17.80(0.29) & 0.84 \\
                                                            & DKS     & \textbf{16.71(0.17)} & \textbf{1.93} \\ \hline
 \multirow{3}{*}{\shortstack[l]{MobileNet}}          & baseline & 23.60(0.22) & -    \\
                                                            & DS       & 22.98(0.17) & 0.62 \\
                                                            & DKS     & \textbf{21.26(0.16)} & \textbf{2.34} \\
\hline
\end{tabular}
}
\end{center}
\vskip -0.1 in
\caption{Accuracy comparison on the CIFAR-100 dataset. For each network, we run each method 5 times and report `mean(std)' error rates. Our method achieves state-of-the-art accuracy when training each backbone network.}
\label{1}
\vskip -0.1 in
\end{table}

\subsection{Experiments on ImageNet}

ImageNet classification dataset~\cite{ref21} is much larger than CIFAR-100 dataset. It has about 1.2 million training images and 50 thousand validation images, consisting of 1000 object classes. For training, images are resized to $256\times256$ first, and then $224\times224$ crops are randomly sampled from the resized images or their horizontal flips normalized with the per-channel mean and std values. For evaluation, we report Top-1 and Top-5 error rates using center crops of the resized validation data.

\textbf{Backbone Networks and Implementation Details.} We use popular ResNets as the backbone networks for evaluation. Specifically, ResNet-18, ResNet-50 and ResNet-152 are considered. All models are trained with SGD for 100 epochs. We set the batch size as 256, the weight decay as 0.0001 and the momentum as 0.9. The learning rate starts at 0.1, and is divided by 10 every 30 epochs. To show the compatibility of DKS with data augmentation methods, we train ResNet-18 and ResNet-50 with a simple data augmentation method, and train ResNet-152 with a more aggressive data augmentation method as in~\cite{ref12}. For each network, we add two auxiliary classifiers after the block Conv3\_x and Conv4\_x. The auxiliary classifiers are constructed with the same building block as in the backbone network. The differences are the number of residual blocks and the number of convolutional filters (see supplementary materials for details). All models are trained on a sever using 8 GPUs.

\textbf{Results Comparison.} Table~\ref{2} shows the results. Similar to the results on the CIFAR-100 dataset, on the ImageNet classification dataset, DS also shows minor accuracy improvements over the baseline models, even using our designed auxiliary classifiers. Its gain in Top-1/Top-5 accuracy is $0.60\%/0.33\%$, $0.38\%/0.11\%$ and $0.46\%/0.25\%$ for ResNet-18, ResNet-50 and ResNet-152, respectively. These results are consistent with the results reported in~\cite{ref12}. Benefiting from the proposed synergy loss, DKS achieves the best results which outperform DS with a margin of $1.78\%/1.25\%$, $1.56\%/1.07\%$ and $1.01\%/0.41\%$ in Top-1/Top-5 accuracy, respectively. Even using simple data augmentation, the ResNet-18/ResNet-50 model trained by our method shows $1.75\%/0.48\%$ Top-1 accuracy gain against the models released at Facebook github\footnote{\url{https://github.com/facebook/fb.resnet.torch}}, which are trained with much stronger data augmentations. Furthermore, it can be seen that the accuracy improvement from our method decreases slightly as network depth increases. Curves of Top-1 training and test error rates can be found in supplemental materials.
\begin{table}
\begin{center}
\resizebox{.95\linewidth}{!}{
\begin{tabular}{p{1.5cm}p{1.5cm}p{3.2cm}p{1.5cm}c}
\hline
Model&Method&Top-1/Top-5 Error(\%)&Gain(\%)\\
\hline
 \multirow{3}{*}{ResNet-18}    & baseline & 31.06\,/\,11.13 & -         \\
                               & DS   & 30.46\,/\,10.80 & 0.60\,/\,0.33 \\
                               & DKS     & \textbf{28.68\,/\,9.55} & \textbf{2.38\,/\,1.58} \\ \hline
 \multirow{3}{*}{ResNet-50}    & baseline & 25.47\,/\,7.58 & -         \\
                               & DS       & 25.09\,/\,7.47 & 0.38\,/\,0.11 \\
                               & DKS     & \textbf{23.53\,/\,6.40} & \textbf{1.94\,/\,1.18} \\ \hline
 \multirow{3}{*}{ResNet-152}   & baseline & 22.45\,/\,5.94 & -         \\
                               & DS       & 21.99\,/\,5.69 & 0.46\,/\,0.25 \\
                               & DKS     & \textbf{20.98\,/\,5.28} & \textbf{1.47\,/\,0.66} \\
\hline
\end{tabular}
}
\end{center}
\vskip -0.1 in
\caption{Accuracy comparison on the ImageNet dataset. }
\label{2}
\vskip -0.2 in
\end{table}

\subsection{Ablation Study}

\textbf{Analysis of Auxiliary Classifiers.} Given a backbone network, the questions of how to design auxiliary classifiers and where to place them are critically important for the deeply-supervised learning methods~\cite{ref22,ref12} and our method. We perform experiments on the ImageNet classification dataset with ResNet-18 to study these two questions. To the first question, we compare our designed auxiliary classifiers and the relatively simple ones suggested in~\cite{ref22}. In the experiments, auxiliary classifiers are added on top of the block Conv3$\_x$ and Conv4$\_x$. With simple auxiliary classifiers, DS introduces $1.17\%/0.80\%$ drop in Top-1/Top-5 accuracy. Comparatively, with our designed auxiliary classifiers, DS brings $0.60\%/0.33\%$ increase and DKS achieves $2.38\%/1.58\%$ gain. The training and test curves are shown in Fig.~\ref{fig:03}. We also perform extensive experiments on the CIFAR-100 dataset using ResNet-32 to analyze the effect of auxiliary classifiers with different levels of complexity to DS and our method. Results are shown in Table~\ref{6}. With very simple auxiliary classifiers, DS shows accuracy drop and DKS further decreases model accuracy. Along with the increased complexity of auxiliary classifiers, DKS outperforms DS with improved margin. Please see supplementary materials for details. To the second question, we consider different settings by adding our designed auxiliary classifiers to at most three intermediate layer locations (including the block Conv2$\_x$, Conv3$\_x$ and Conv4$\_x$) of ResNet-18. Detailed results are shown in Table~\ref{3} where C1, C2, C3 and C4 denote the auxiliary classifiers connected on top of the last layer, the block Conv4\_x, Conv3\_x and Conv2\_x, sequentially. From Table~\ref{3}, we can make following observations:\begin{table}[htp!]
\begin{center}
\resizebox{.9\linewidth}{!}{
\begin{tabular}{p{2.2cm}p{1.4cm}p{1.4cm}p{1.9cm}c}
\hline
Aux.Classifiers & Error(\%)   (DS) & Error(\%)  (DKS) & Avg Gain(\%) (DKS to DS)\\
\hline
 AP+2FC & 31.85(0.42) & 35.09(0.54)& -3.24 \\
 AP+1Conv+2FC   & 30.24(0.05) & 32.52(0.27) & -2.28 \\
 Narrow Blocks & 29.52(0.30) & 29.18(0.28) & 0.34 \\
 Shallow Blocks & 29.39(0.09) & 28.69(0.28) & 0.70 \\
 Ours & 29.89(0.26) & 26.81(0.36) & 3.08 \\
\hline
\end{tabular}
}
\end{center}
\vskip -0.1 in
\caption{Accuracy comparison of DKS and KD using auxiliary classifiers with different levels of complexity. The baseline ResNet-32 model trained on CIFAR-100 shows 29.97\%(0.33\%) `mean(std)' error rates over 5 runs. In the table, AP denotes average pooling layer, Conv denotes convolutional layer and FC denotes fully connected layer.}
\label{6}
\end{table}
\begin{table}[htp!]
\begin{center}
\resizebox{.8\linewidth}{!}{
\begin{tabular}{p{1.6cm}p{3.2cm}p{1.3cm}c}
\hline
Model&Top-1/Top-5 Error(\%)&Gain(\%)\\
\hline
 baseline($C_1$)   & 31.06\,/\,11.13 & -         \\
 $C_1C_2$          & 29.64\,/\,10.09 & 1.42\,/\,1.04 \\
 $C_1C_3$          & 29.30\,/\,9.86 & 1.76\,/\,1.27 \\
 $C_1C_4$          & 29.36\,/\,9.91 & 1.70\,/\,1.22 \\
 $C_1C_2C_3$       & \textbf{28.68\,/\,9.55} & \textbf{2.38\,/\,1.58} \\
 $C_1C_2C_3C_4$    & 29.00\,/\,9.79 & 2.06\,/\,1.34 \\
\hline
\end{tabular}
}
\end{center}
\vskip -0.1 in
\caption{Accuracy gains of DKS with auxiliary classifiers connected to different intermediate layers of ResNet-18.}
\label{3}
\end{table}
\begin{table}[htp!]
\begin{center}
\resizebox{.8\linewidth}{!}{
\begin{tabular}{p{1.5cm}p{3.2cm}p{1.5cm}c}
\hline
Model&Top-1/Top-5 Error(\%)&Gain(\%)\\
\hline
 $C_1$      & 31.06\,/\,11.13 & 2.38\,/\,1.58 \\
 $C_2$      & 30.69\,/\,11.05 & 3.23\,/\,2.16 \\
 $C_3$      & 31.89\,/\,11.51 & 2.39\,/\,1.68 \\
\hline
\end{tabular}
}
\end{center}
\vskip -0.1 in
\caption{Accuracy gains of DKS to each individual auxiliary classifier connected to the corresponding intermediate layer of ResNet-18.}
\label{4}
\vskip -0.2 in
\end{table}(1) With only one auxiliary classifier, an early location is better than a relatively deep location; (2) Adding two or all of three auxiliary classifiers obtains more large gain than adding only one; (3) Adding C4 connected to an earlier intermediate layer into the combination of C2 and C3 decreases its accuracy. According to these results, we choose to add C2 and C3 for all experiments on the ImageNet classification dataset. In addition, we also analyze whether DKS is beneficial to auxiliary supervision branches or not. To this end, we train each individual auxiliary supervision branch separately, and compare it with the corresponding one trained with DKS. According to the results shown in Table~\ref{4}, we can see that our method also brings obvious accuracy gain to each auxiliary supervision branch.

\textbf{Comparison of Knowledge Matching Strategies.} We also compare the performance of three pairwise knowledge matching strategies shown in Fig.~\ref{fig:02}. Experiments are conducted on the ImageNet classification dataset with ResNet-18 using our best auxiliary classifier setting just discussed. Compared with the baseline model, our method obtains $0.50\%/0.45\%$, $2.22\%/1.19\%$ and $2.38\%/1.58\%$ increase in Top-1/Top-5 accuracy by using the top-down, bottom-up and bi-directional pairwise knowledge matching strategies, respectively. As the bi-directional strategy shows the best results, we adopt it as the default choice for DKS. Another interesting observation is that they all achieve improved results compared with the baseline method, showing that the pairwise knowledge transfer among the supervised classifiers connected to the backbone network is really helpful in regularizing model training.

\textbf{DKS on Very Deep Network.} Next, we conduct a set of experiments to analyze the performance of DKS on very deep CNNs. In the experiments, we consider the training of a ResNet variant with 1202 layers~\cite{ref15} on the CIFAR-100 dataset. Unlike auxiliary classifiers used in the other experiments, we study DKS with shallow but wide auxiliary classifiers in this experiment (see supplementary materials for details). Remarkably, although the network depth is significantly increased, the average accuracy of the models trained with our method is $69.54\%$, showing a $3.76\%/2.04\%$ margin compared with the baseline/DS method.

\textbf{DKS with Strong Regularization.} In order to explore the compatibility of DKS and more strong regularization methods, we conduct the experiments on the CIFAR-100 dataset following~\cite{ref16}. We add a dropout layer with a ratio of 0.3 after the first layer of every building block of WRN-28-10. The results are shown in Table~\ref{1}. It can be seen that the models trained with DKS show a mean accuracy of $16.71\%$, bringing $0.53\%$ gain to the case without dropout.

\textbf{DKS vs. Knowledge Distillation.} Further, we compare the performance of DKS, Knowledge Distillation (KD) and its variants. Experiments are conducted on the ImageNet classification dataset using ResNet-18. We use a pre-trained ResNet-50 model as the teacher and consider three different KD settings: (1) KD on C1 (the standard KD as in~\cite{ref37}); (2) KD on C1+DS; (3) KD on C2C3+DS. We evaluate temperature values of [1, 2, 5, 10, 20] and choose the best choice for each KD setting. From the results shown in Table~\ref{5}, we can make following observations: (1) KD can improve model training in all cases; (2) Distilling learnt knowledge into auxiliary classifiers connected to the earlier layers has small gain to DS, and more large gain can be achieved by applying KD on auxiliary classifiers added to the deep layers; (3) DKS achieves the best performance, showing the effectiveness of the proposed synergy loss.

\textbf{DKS on Noisy Data.} Finally, we explore the capability of our method to handle noisy data. Following~\cite{ref13}, we use CIFAR-10 dataset and DenseNet (d=40, k=12) as a test case. Before training, we randomly sample a fixed ratio of training data and replace their ground truth labels with randomly generated wrong labels. Results show that the average accuracy of the baseline model decreases from $94.62\%$ to $82.07\%$, while DS further decreases it to $80.47\%$ and ours is $83.73\%$, when $50\%$ training data are corrupted. As the ratio of the corrupted training data goes to $80\%$, our model still has $67.19\%$ mean\begin{table}[htp!]
\begin{center}
\resizebox{.85\linewidth}{!}{
\begin{tabular}{p{2.6cm}p{3.1cm}p{1.3cm}c}
\hline
Model&Top-1/Top-5 Error(\%)&Gain(\%)\\
\hline
 baseline & 31.06\,/\,11.13 & -         \\
 DS   & 30.46\,/\,10.80 & 0.60\,/\,0.33 \\
 KD on C1~\cite{ref37} & 29.71\,/\,10.33 & 1.35\,/\,0.80 \\
 KD on C1+DS & 29.38\,/\,10.10 & 1.68\,/\,1.03 \\
 KD on C2C3+DS & 30.32\,/\,10.64 & 0.74\,/\,0.49 \\
 DKS     & \textbf{28.68\,/\,9.55} & \textbf{2.38\,/\,1.58} \\
\hline
\end{tabular}
}
\end{center}
\vskip -0.2 in
\caption{Accuracy comparison of DKS, KD and its variants on the ImageNet classification dataset using ResNet-18.}
\label{5}
\vskip -0.2 in
\end{table} accuracy, outperforming the baseline/DS with a margin of $2.51\%/2.27\%$. These experiments partially show that our method has good capability to suppress noise disturbance and behaves like a strong regularizer.

\subsection{Discussion}
Although the CNNs used in our experiments have sophisticated building block designs which increase the flexibility of feature connection path and show stable convergence, our DKS can impressively improve their training in comparison to the standard training scheme and DS. This is first benefitted from adding proper auxiliary classifiers to the intermediate layers of the network, but we believe it is more benefitted from the proposed synergy loss which enables comprehensive pairwise knowledge matching among all supervised classifiers connected to the network, enhancing learnt feature representation. On the other hand, we observe substantial time increase for model training. For instance, a baseline ResNet-18 model is trained for about 20 hours on a server with 8 GPUs (an SSD is used to accelerate data accessing process), while our method needs about 37 hours, nearly doubling the training time. Besides, the training time for DS is almost the same as our method. We believe this mainly correlates with the number of auxiliary classifiers and their complexity. Therefore, there is a tradeoff between the required training time and the expected accuracy improvement. Achieving larger accuracy gain needs auxiliary classifiers to be more complex, while simple ones usually worsen model accuracy. Since increasing the number of auxiliary classifiers does not always bring more high accuracy gain, as shown in our ablation study, we think the current increase in training time is reasonable. More importantly, all auxiliary classifiers are discarded at inference phase, thus there is no extra computational cost.

\section{Conclusion}
In this paper, we revisit the deeply-supervised learning research and propose a new optimization scheme called DKS for training deep CNNs. It introduces a novel synergy loss which regularizes the training by considering dense pairwise knowledge matching among all supervised classifiers connected to the network. Extensive experiments on two well-known image classification tasks validate the effectiveness of our method.

{\small
\bibliographystyle{ieee}
\bibliography{DKS_3487_camera_ready}
}

\clearpage

\appendix
  \renewcommand{\appendixname}{Appendix~\Alph{section}}
  \section{Appendix 1: Theoretical Insights}
In this section, we give some theoretical insights about why the synergy mechanism in our DKS method shows better performance than the Deeply-Supervised (DS) learning method~\cite{ref22,ref12}. For simplicity, we focus on the optimization of a CNN regression model with one auxiliary branch. Inspired by~\cite{ref52}, we give a formal proof for that the pair-wise synergy term behaves as a regularizer which penalizes the inconsistency between the gradients of the two branches w.r.t. their shared intermediate feature map. Such a proof can be generalized to the optimization of a deep CNN classification model.

Suppose we need to train a CNN model to fit target data distribution $x,y\sim D$ where $x$ denotes the input and $y$ denotes the expected output. The model has two output heads, the top-most one $\hat{y}_1$ and the auxiliary one $\hat{y}_2$ as illustrated in Fig.~\ref{fig:regression_model}. The forward process is as follows, $z = f(x),~\hat{y}_1 = g_1(z),~\hat{y}_2 = g_2(z)$, where $z$ denotes the intermediate feature map. In DS configuration, the loss function used to guide the training process can be written as
\begin{align*}
  &L_{DS} = \mathop{\mathbb{E}}\limits_{x,y \sim D,~\epsilon}\left(\frac{1}{2}\left|\left|g_1(z+\epsilon)-y\right|\right|^2 \right. \\
  &~~~ \left. + \frac{1}{2}\left|\left|g_2(z+\epsilon)-y\right|\right|^2\right),
\end{align*} where $\epsilon$ denotes the random perturbations on the intermediate feature map caused by data augmentation. Here, we assume $\mathop{\mathbb{E}}(\epsilon) = 0,~\mathop{\mathbb{E}}(\epsilon^2) = \sigma^2$. In DKS configuration, the loss function can be written as
\begin{align*}
&L_{DKS} = \mathop{\mathbb{E}}\limits_{x,y \sim D,~\epsilon}\left( \frac{1}{2}\left|\left|g_1(z+\epsilon)-y\right|\right|^2 \right. \\
&~ \left. + \frac{1}{2}\left|\left|g_2(z+\epsilon)-y\right|\right|^2 + \frac{1}{2}\left|\left|g_1(z+\epsilon)-g_2(z+\epsilon)\right|\right|^2 \right).
\end{align*}
\begin{proposition}
The synergy term in DKS guarantees more consistent gradients of the two branches w.r.t. $z$ compared with DS. That is, the synergy term penalizes $\left|\left|\frac{\partial \hat{y}_1}{\partial z} - \frac{\partial \hat{y}_2}{\partial z}\right|\right|^2$.
\end{proposition}
\begin{proof}
The synergy term in DKS:
\begin{align*}
&\mathop{\mathbb{E}}\limits_{x,y \sim D,~\epsilon}\left(\frac{1}{2}\left|\left|g_1(z+\epsilon)-g_2(z+\epsilon)\right|\right|^2\right)\\
&= \mathop{\mathbb{E}}\limits_{x,y \sim D,~\epsilon}\left(\frac{1}{2}\left|\left|g_1(z) + \epsilon \frac{\partial g_1(z)}{\partial z} + O(\epsilon^2)\right.\right.\right.\\ &~~~- \left.\left.\left.g_2(z) - \epsilon \frac{\partial g_2(z)}{\partial z} + O(\epsilon^2)\right|\right|^2\right)\\
&=\frac{1}{2}\mathop{\mathbb{E}}\limits_{x,y \sim D}\left(\left|\left|g_1(z) - g_2(z)\right|\right|^2\right) \\ &~~~+ \mathop{\mathbb{E}}(\epsilon)~\mathop{\mathbb{E}}\limits_{x,y \sim D} \left(\left|\left|g_1(z) - g_2(z)\right|\right| \cdot \left|\left|\frac{\partial g_1(z)}{\partial z} - \frac{\partial g_2(z)}{\partial z} \right|\right|\right) \\ &~~+ \frac{1}{2} \mathop{\mathbb{E}}(\epsilon^2)~\mathop{\mathbb{E}}\limits_{x,y \sim D}\left(\left|\left|\frac{\partial g_1(z)}{\partial z} - \frac{\partial g_2(z)}{\partial z} \right|\right|^2\right) + O(\epsilon^4)
\\ &= \frac{1}{2}\mathop{\mathbb{E}}\limits_{x,y \sim D}\left(\left|\left|g_1(z) - g_2(z)\right|\right|^2\right) \\ &~~+ \frac{1}{2}\sigma^2 \mathop{\mathbb{E}}\limits_{x,y \sim D}\left(\left|\left|\frac{\partial g_1(z)}{\partial z} - \frac{\partial g_2(z)}{\partial z} \right|\right|^2\right) + O(\epsilon^4)
\end{align*}
\end{proof}
\begin{figure}
\centering
\includegraphics [width=0.4\textwidth]{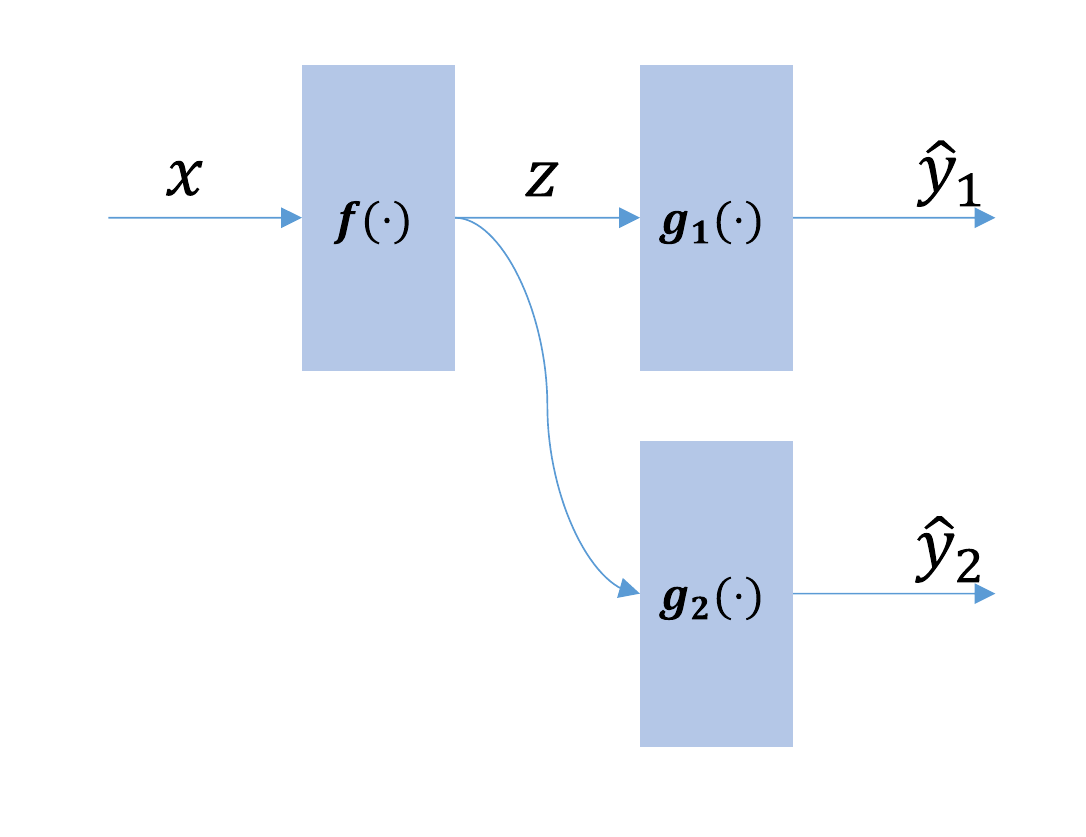}
\caption{The regression model used in the proof.}
\label{fig:regression_model}
\end{figure}
\section{Appendix 2: Design Details of Auxiliary Classifiers}
As we described in the paper, we append carefully designed auxiliary classifiers on top of some intermediate layers of a given backbone network when applying our DKS method to CIFAR-100 and ImageNet classification datasets. In this section, we provide the design details of the auxiliary classifiers.
\subsection{Auxiliary Classifiers for CIFAR-100}
On the CIFAR-100 dataset, we test several kinds of backbone networks including ResNets~\cite{ref15}, DenseNets~\cite{ref17}, WRNs~\cite{ref16} and MobileNet~\cite{ref27}. In this sub-section, we describe the auxiliary classifiers used in the Section 4.1, the Section 4.3 `Analysis of Auxiliary Classifiers', `DKS on Very Deep Network', `DKS with Strong Regularization' and `DKS on Noisy Data' of our paper, respectively.
\subsubsection{Auxiliary Classifiers Used in Section 4.1}
\paragraph{Locations.}
In the experiments, we add two auxiliary classifiers to every backbone network. Their locations for different backbone networks are shown in Fig.\ref{fig:locations_cifar}.
\paragraph{Structures.}
In the experiments, we append relatively complex auxiliary supervision branches on top of certain intermediate layers during network training. Specifically, every auxiliary supervision branch is composed of the same building block (e.g., residual block in ResNet) as in the backbone network. The differences lie in the numbers and parameter sizes of convolutional layers. As empirically verified in~\cite{ref23}, early layers lack coarse-level features which are helpful for image-level classification. In order to address this problem, we use a heuristic principle making the paths from the input to all classifiers have the same number of down-sampling layers. We detail the hyper-parameter settings of the convolutional layers of different backbone networks in Table ~\ref{tab:CIFAR_structures_resnet}, Table~\ref{tab:CIFAR_structures_wrn}, Table~\ref{tab:CIFAR_structures_densenet} and Table~\ref{tab:CIFAR_structures_mobilenet}, respectively.

\begin{figure}
\centering
\includegraphics [width=0.5\textwidth]{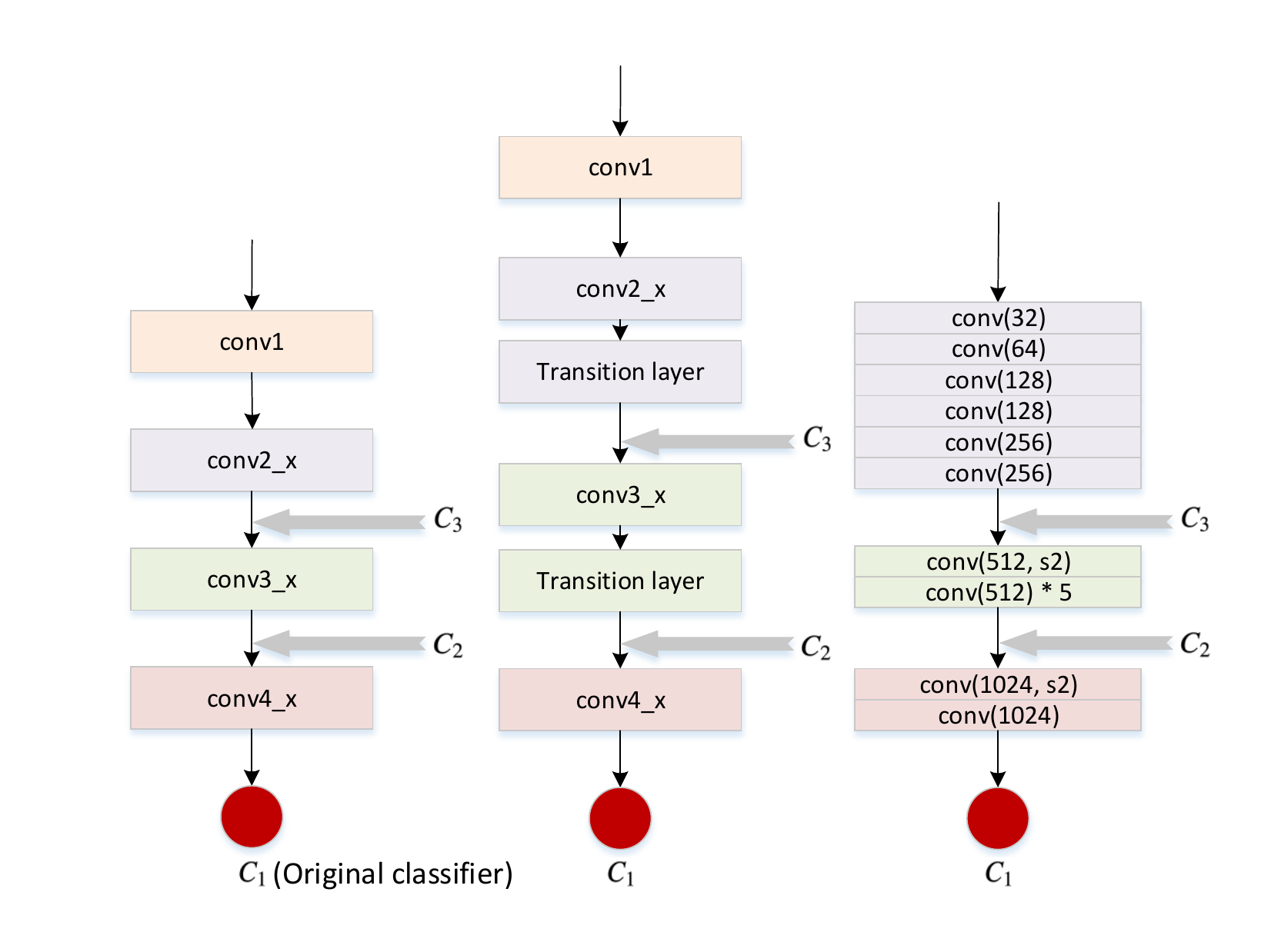}
\caption{Locations of the auxiliary classifiers added to the backbone networks evaluated on the CIFAR-100 dataset. The left figure is for ResNets and WRNs, and the middle figure is for DenseNets, and the right figure is for MobileNet. The grey thick arrows indicate the locations where auxiliary classifiers are added. We denote these three classifiers as $C_1$, $C_2$ and $C_3$ respectively, where $C_1$ is the original classifier of the backbone network.}
\label{fig:locations_cifar}
\end{figure}

\subsubsection{Auxiliary Classifiers Used in Section 4.3 `Analysis of Auxiliary Classifiers'}
In order to analyze the impact of the complexity of auxiliary classifiers, we evaluate different auxiliary classifier designs for ResNet-32 in the Section 4.3 `Analysis of Auxiliary Classifiers'.
\paragraph{Locations.}
The locations are the same as that described in Fig.~\ref{fig:locations_cifar}.
\paragraph{Structures.}
As described in the Table 3 of our paper, we evaluated four more types of auxiliary classifiers besides our final design. AP+2FC refers to one average pooling layer + two fully connected layers. AP+1Conv+2FC refers to one average pooling layer + one convolutional layer + two fully connected layers. Narrow Blocks means the auxiliary classifiers are narrower than the original design (i.e., the top-most classifier connected to the last layer of the backbone network). Shallow Blocks means the auxiliary classifiers are shallower than the original design. Please refer to Fig.~\ref{fig:ep_study_FC} and Table~\ref{tab:ep_study_FC} for more details.

\begin{figure}
\centering
\includegraphics [width=0.425\textwidth]{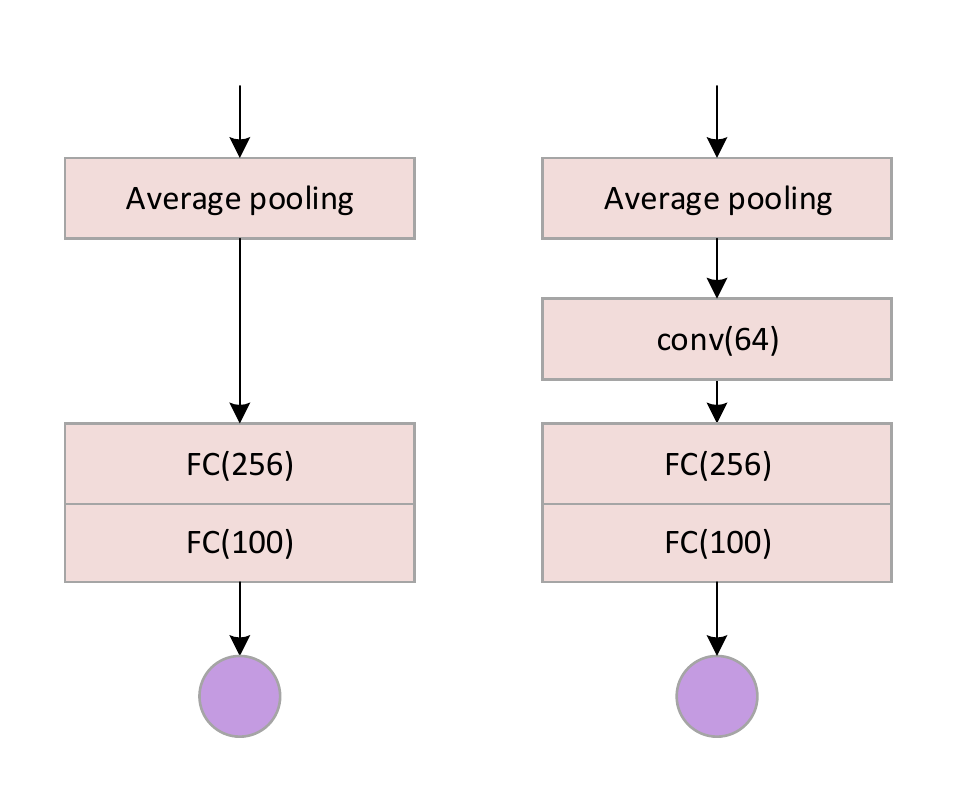}
\caption{Details of the AP+2FC and AP+1Conv+2FC auxiliary classifiers added to the ResNet-32 backbone network evaluated on the CIFAR-100 dataset. The output size of the average pooling layer is $4 \times 4$. The number of output channels of every layer is shown in the parentheses.}
\label{fig:ep_study_FC}
\end{figure}

\subsubsection{Auxiliary Classifiers Used in Section 4.3 `DKS on Very Deep Network'}
We conduct a set of experiments to analyze the performance of DKS on very deep CNNs. In the experiments, we consider the training of a ResNet variant with 1202 layers~\cite{ref15} on the CIFAR-100 dataset. Unlike auxiliary classifiers used in the other experiments, we study DKS with shallow but wide auxiliary classifiers in the experiments.
\paragraph{Locations.}
In the experiments, we add three auxiliary classifiers to the ResNet-1202 backbone network. Their locations are shown in Fig.\ref{fig:locations_resnet1202}.
\paragraph{Structures.}
The auxiliary classifiers added to the ResNet-1202 backbone network have the macro-structures defined in Fig.\ref{fig:structure_resnet1202}. The number of convolutional layers for every auxiliary classifier can be found in Table \ref{tab:structure_resnet1202}.

\begin{figure}
  \begin{subfigure}{.24\textwidth}
    \centering
    \includegraphics [width=0.95\linewidth]{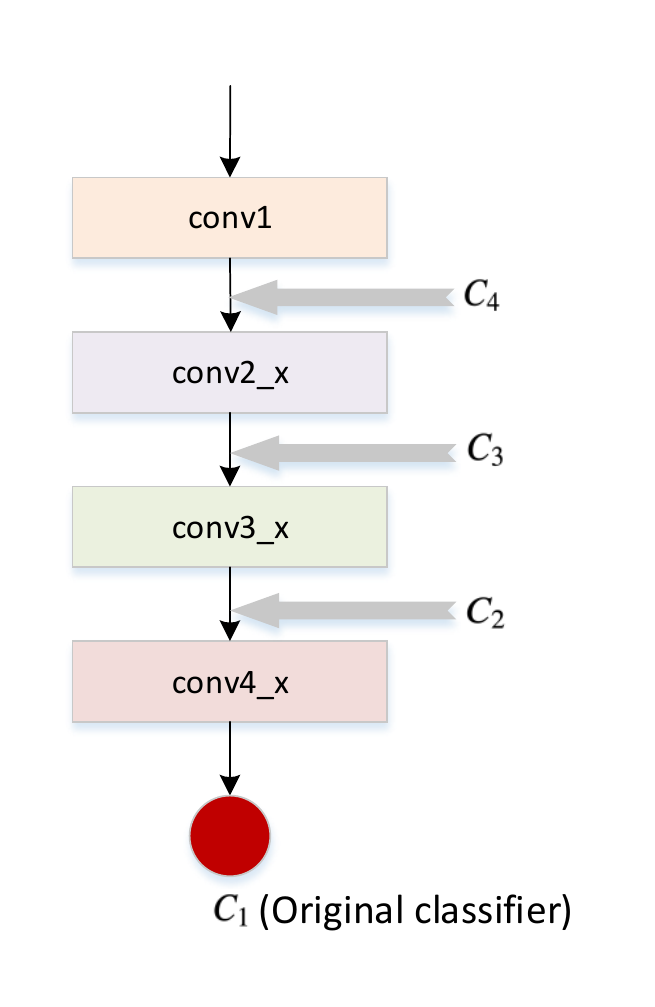}
    \caption{}
    \label{fig:locations_resnet1202}
  \end{subfigure}%
  \begin{subfigure}{.24\textwidth}
    \centering
    \includegraphics [width=0.95\linewidth]{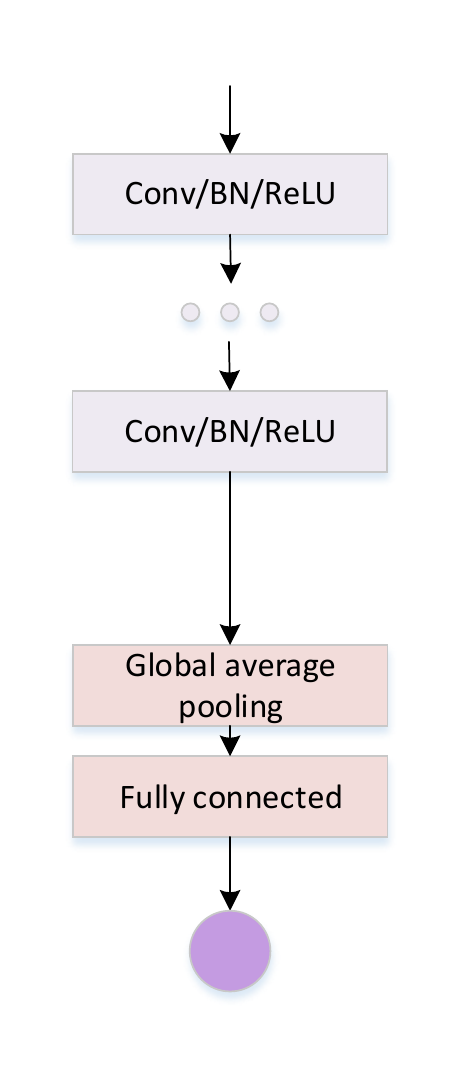}
    \caption{}
    \label{fig:structure_resnet1202}
  \end{subfigure}
\caption{(a) Locations of the auxiliary classifiers added to the ResNet-1202 backbone network evaluated on the CIFAR-100 dataset. The grey thick arrows indicate the layer locations where auxiliary classifiers are added. We denote these three auxiliary classifiers as $C_2$, $C_3$ and $C_4$, respectively. (b) Structure of the auxiliary classifiers. All the convolutional layers in this structure have the same kernel size (= 3$\times$3) and the same stride (= 1), but have different number of filters (yielding different number of output channels). The numbers of convolutional layers and the corresponding filters for every auxiliary classifier can be found in Table \ref{tab:structure_resnet1202}.}
\end{figure}

\begin{table}[]
\centering
\begin{tabular}{|c|c|c|c|}
\hline
& $C_2$ & $C_3$ & $C_4$\\
\hline
conv1 & 512 & 256 & 128\\
\hline
conv2 & -   & 512 & 256\\
\hline
conv3 & -   & -   & 512\\
\hline
\end{tabular}
\caption{Details of the convolutional layers of the auxiliary classifiers added to the ResNet-1202 backbone network evaluated on the CIFAR-100 dataset. In the table, the number in every cell indicates how many filters are in this convolutional layer. For example, $C_3$ has two convolutional layers where the first layer has 256 filters and the second layer has 512 filters.}
\label{tab:structure_resnet1202}
\end{table}

\subsection{Auxiliary Classifiers for ImageNet}
On the ImageNet classification dataset, we use popular ResNet-18, ResNet-50 and ResNet-152~\cite{ref15} as the backbone networks. In this sub-section, we describe the auxiliary classifiers used in the Section 4.2 and some experiments of Section 4.3 of our paper.
\paragraph{Locations.}
In the experiments, we add at most 3 auxiliary classifiers to every backbone network. Following the definition in our paper, we denote the original classifier (i.e., the top-most classifier added to the last layer of a backbone network) as $C_1$ and the auxiliary classifiers as $C_2, C_3$ and $C_4$, as shown in Fig.\ref{fig:locations_imagenet}. Recall that $C_4$ is an extra auxiliary classifier which is added for analyzing the accuracy effect of the increasing number of the auxiliary classifiers, as described in Table 4 of our paper. For the main experiments in the Section 4.2 of our paper, we add 2 auxiliary classifiers (i.e., $C_2$ and $C_3$) to every backbone network.

\begin{figure}
\centering
\includegraphics [width=0.3\textwidth]{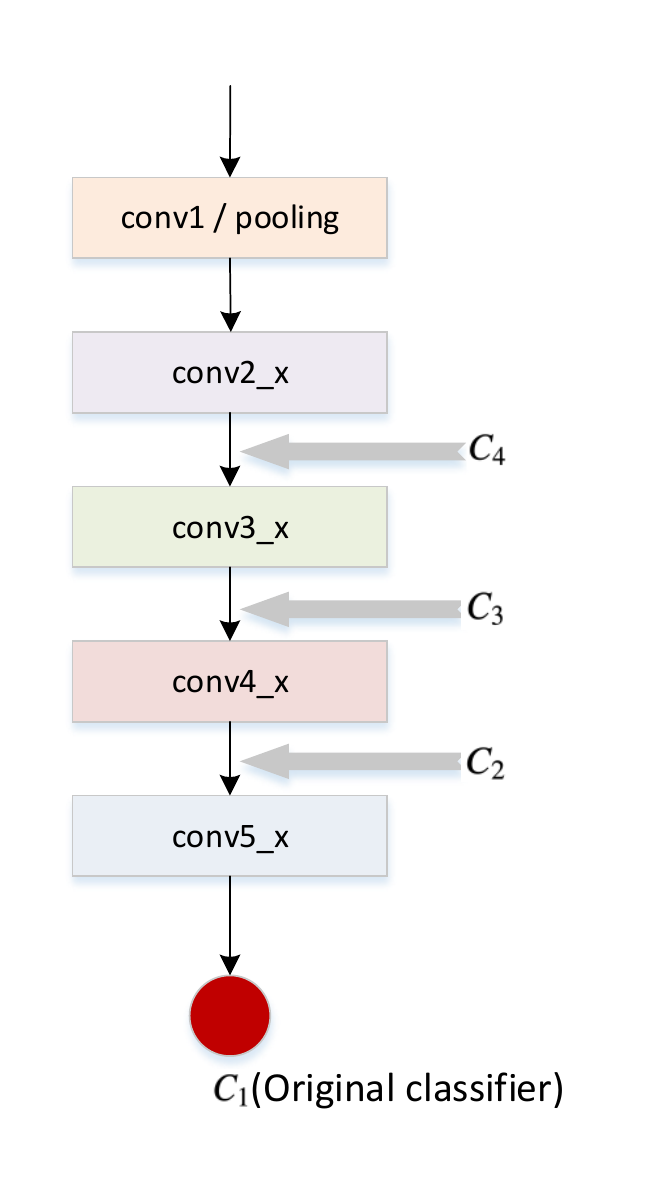}
\caption{Locations of the auxiliary classifiers added to the ResNet backbone networks evaluated on the ImageNet classification dataset. The grey thick arrows indicate the locations where auxiliary classifiers are added. Following the definition in our paper, we denote these three auxiliary classifiers as $C_2$, $C_3$ and $C_4$, respectively.}
\label{fig:locations_imagenet}
\end{figure}

\paragraph{Structures.}
In all of the experiments except for the one regarding Fig.3 `DS with simple aux. classifiers' of our paper, the auxiliary classifiers added to all backbone networks have the same macro-structure. Generally, we design these auxiliary classifiers with the same building blocks as the backbone network. To guarantee that all the paths from the input to different classifier outputs have the same down-sampling process, we design the auxiliary classifiers according to the corresponding building blocks in the backbone network. For example, the auxiliary classifier $C_3$ has its own conv\_4x and conv\_5x blocks acting as down-sampling stages, whose parameter size is smaller than that of the corresponding stages in the backbone network. After these down-sampling stages, there are also a global average pooling layer and a fully connected layer. We show the details of the convolutional blocks of the auxiliary classifiers in Table \ref{tab:structures_imagenet}.

In the experiment regarding Fig.3 `DS with simple aux. classifiers' of our paper, we use a very simple structure as suggested in~\cite{ref22} for the auxiliary classifiers. The structure is shown in Fig.\ref{fig:FC_imagenet}. Specifically, the hyper-parameters of the average pooling layer in $C_2$ are kernel size = $5\times5$, stride = 3 and padding = 1, and in $C_3$ are kernel size = $7\times7$, stride = 7 and padding = 3. The feature map with size of $4\times4\times256$ or $4\times4\times128$ is fed into its respective fully connected layer with a Softmax function for final predication.

\begin{figure}
\centering
\includegraphics [width=0.2\textwidth]{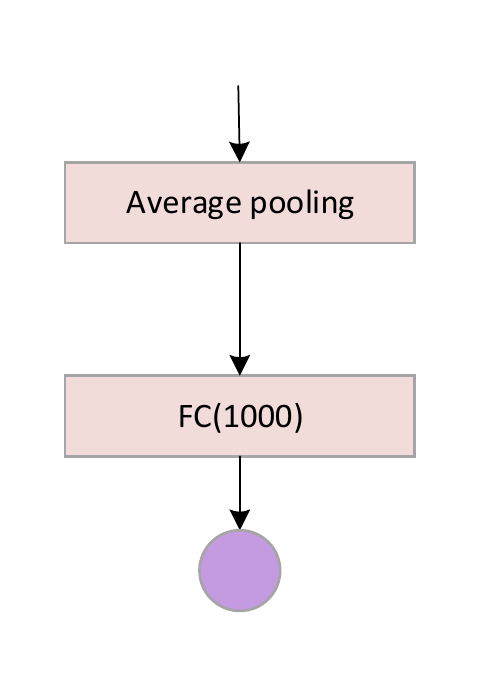}
\caption{Structure of the simple auxiliary classifiers added to the ResNet-18 backbone network evaluated on the ImageNet classification dataset. The average pooling layer will down-sample the input feature map into a new one with the spatial size of 4$\times$4. Then the feature map will be flattened to be a one-dimensional vector which will be fed into the fully connected layer. Here, the purple circle denotes the Softmax layer which will output a probability distribution.}
\label{fig:FC_imagenet}
\end{figure}

\section{Appendix 3: Accuracy Curves of ResNet Models Trained on ImageNet}
Fig.~\ref{fig:curves} shows the curves of Top-1 training error and test error of the ResNet models trained on the ImageNet classification dataset. Compared with the standard training scheme and DS, it can be seen that DKS has worst training accuracy but best test accuracy for all backbone networks, showing better capability to suppress over-fitting during training.

\onecolumn
\begin{figure}
\begin{subfigure}{.33\textwidth}
  \centering
  \includegraphics[width=.95\linewidth]{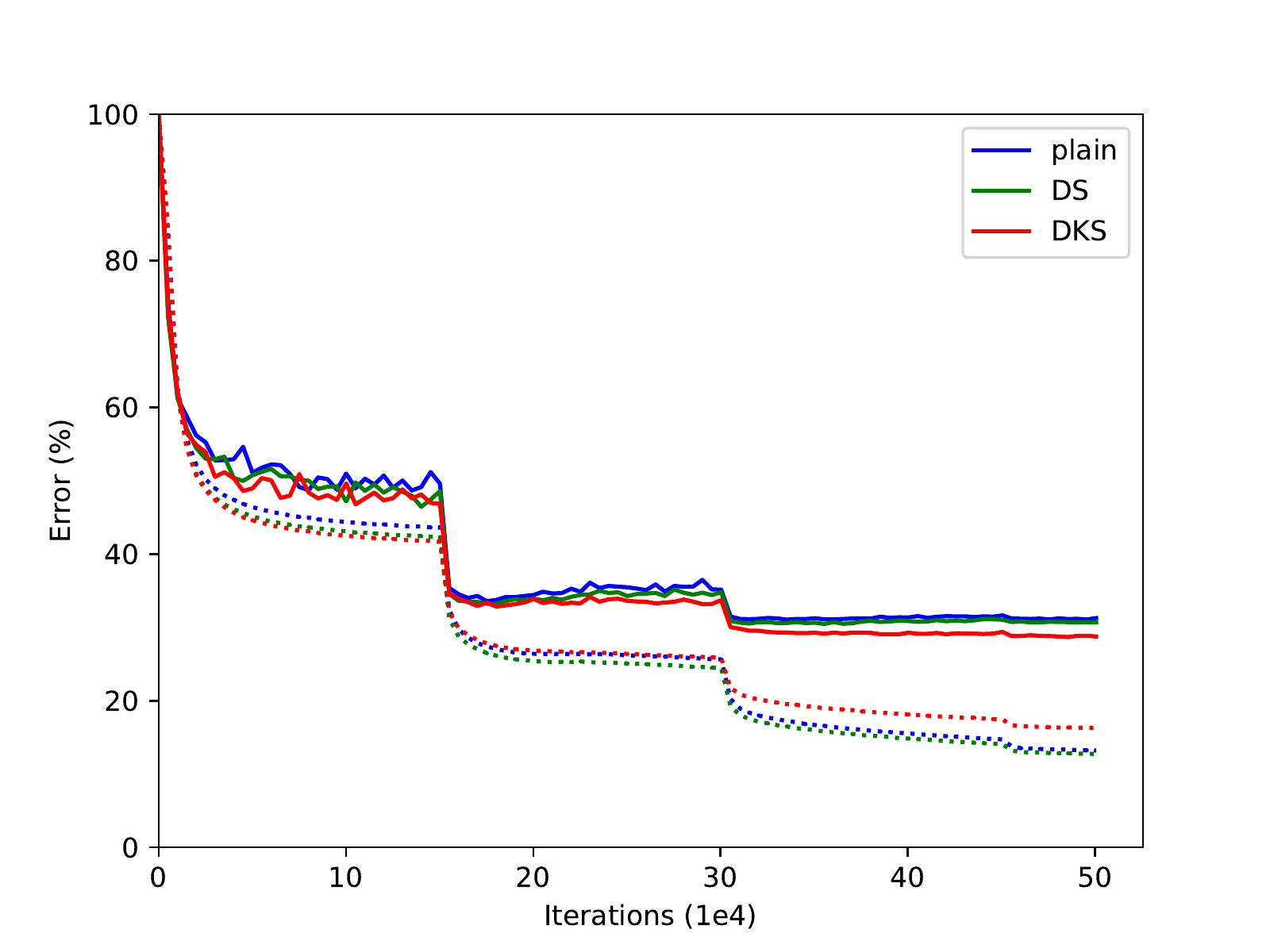}
  \caption{}
\end{subfigure}%
\begin{subfigure}{.33\textwidth}
  \centering
  \includegraphics[width=.95\linewidth]{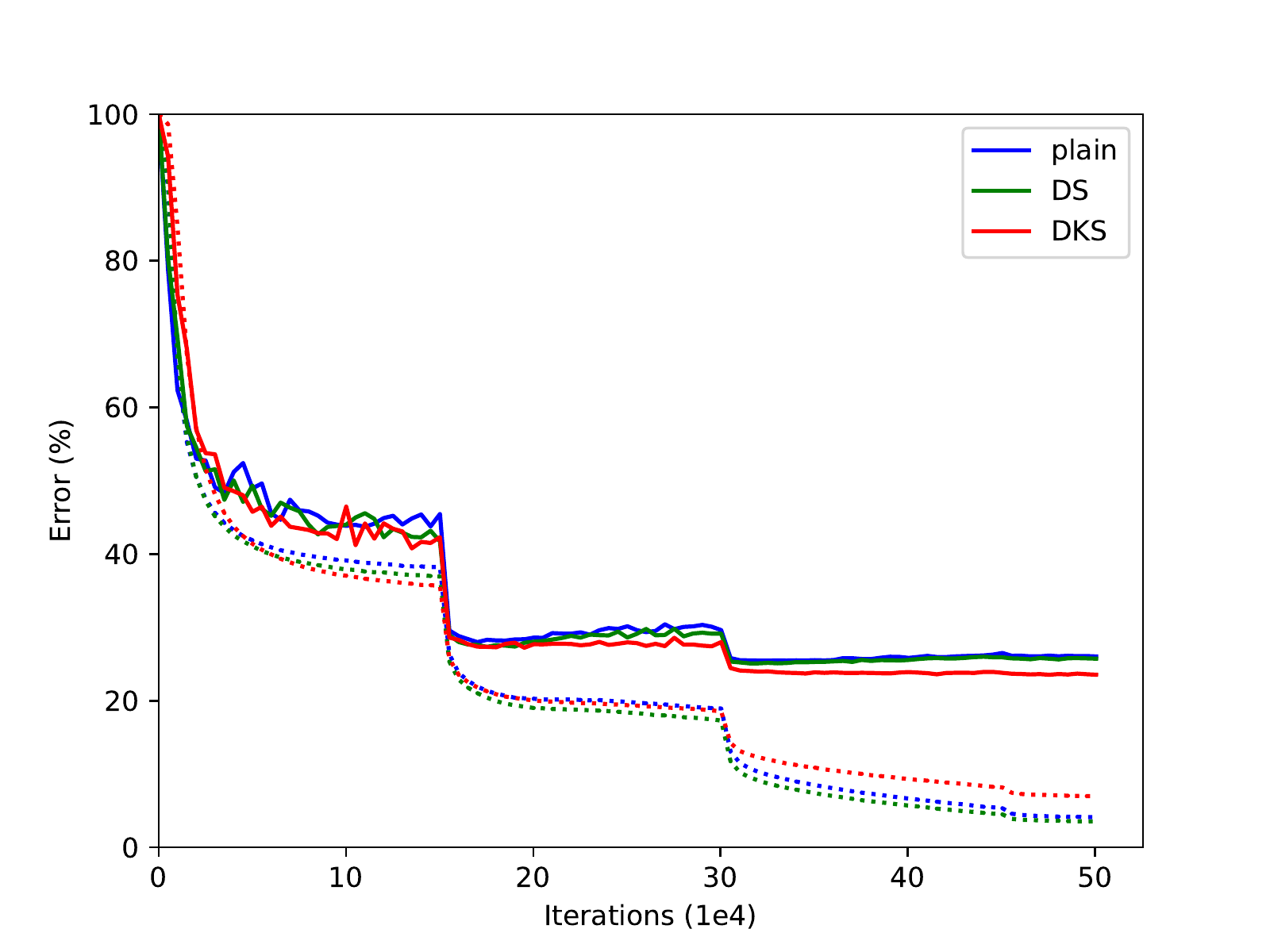}
  \caption{}
\end{subfigure}
\begin{subfigure}{.33\textwidth}
  \centering
  \includegraphics[width=.95\linewidth]{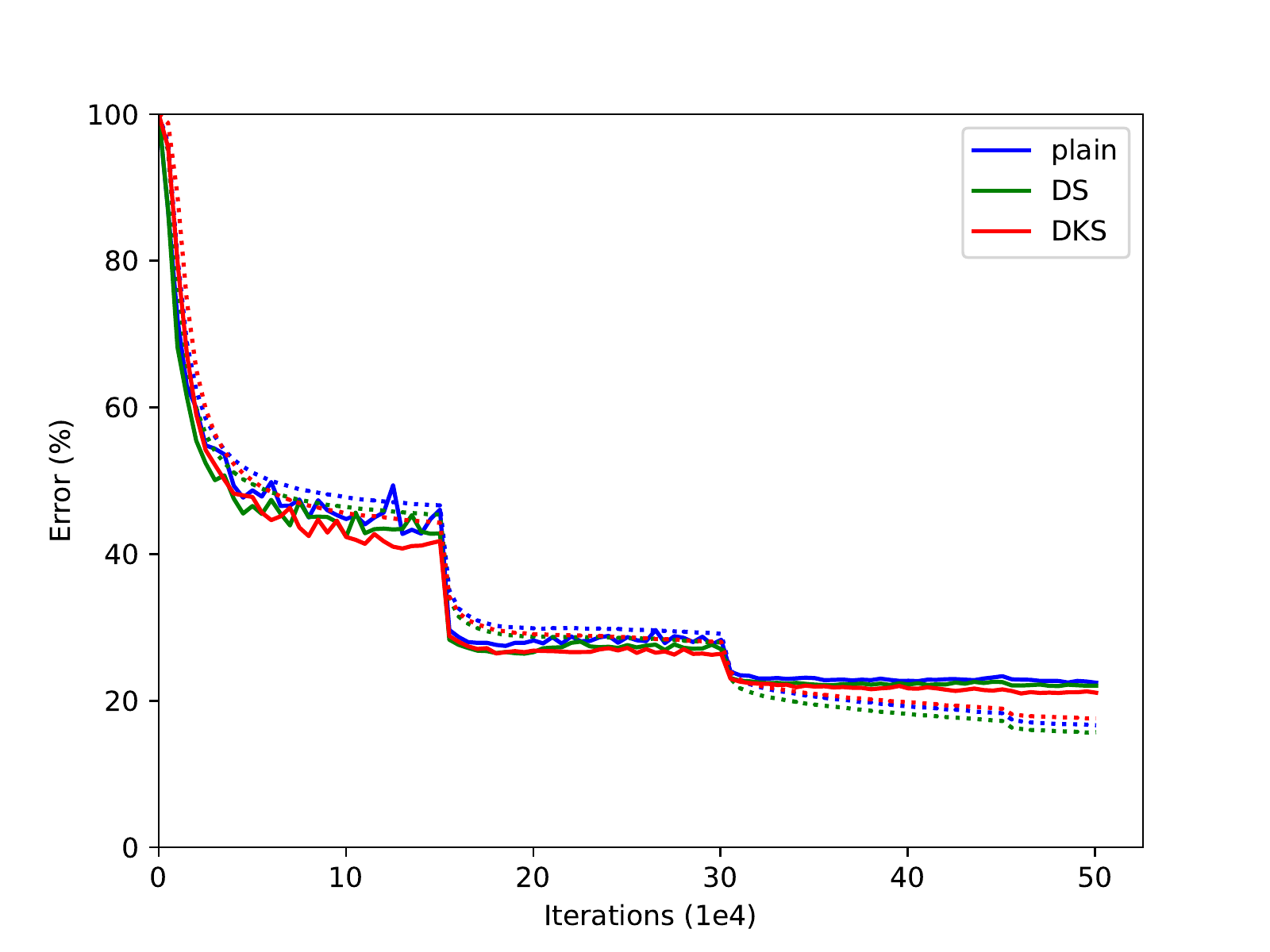}
  \caption{}
\end{subfigure}%
\caption{Curves of Top-1 training error (dashed line) and test error (solid line) on the ImageNet classification dataset with ResNet-18 (a), ResNet-50 (b) and ResNet-152 (c).}
\label{fig:curves}
\vskip 0.1 in
\end{figure}

\begin{table}[]
\centering
\scalebox{0.8}{
\begin{tabular}{|c|c|c|c|c|c|c|}
\hline
 & \multicolumn{3}{c|}{ResNet(d=32)} & \multicolumn{3}{c|}{ResNet(d=110)} \\
\hline
 & $C_1$ & $C_3$ & $C_2$ & $C_1$ & $C_3$ & $C_2$\\
\hline
 conv1 & {$3 \times 3,\,16$} & - & - & {$3 \times 3,\,16$} & - & -\\
\hline
 conv2\_x & {$\begin{bmatrix} 3\times3,\,16 \\ 3\times3,\,16 \end{bmatrix} \times 5$} & - & - & $\begin{bmatrix} 3\times3,\,16 \\ 3\times3,\,16 \end{bmatrix} \times 18$ & - & - \\
\hline
 conv3\_x & {$\begin{bmatrix} 3\times3,\,32 \\ 3\times3,\,32 \end{bmatrix} \times 5$} & {$\begin{bmatrix} 3\times3,\,32 \\ 3\times3,\,32 \end{bmatrix} \times 5$} & - & {$\begin{bmatrix} 3\times3,\,32 \\ 3\times3,\,32 \end{bmatrix} \times 18$} & {$\begin{bmatrix} 3\times3,\,32 \\ 3\times3,\,32 \end{bmatrix} \times 9$} & - \\
\hline
 conv4\_x & {$\begin{bmatrix} 3\times3,\,64 \\ 3\times3,\,64 \end{bmatrix} \times 5$} & {$\begin{bmatrix} 3\times3,\,64 \\ 3\times3,\,64 \end{bmatrix} \times 3$} & {$\begin{bmatrix} 3\times3,\,128 \\ 3\times3,\,128 \end{bmatrix} \times 5$} & {$\begin{bmatrix} 3\times3,\,64 \\ 3\times3,\,64 \end{bmatrix} \times 18$} & {$\begin{bmatrix} 3\times3,\,64 \\ 3\times3,\,64 \end{bmatrix} \times 9$} & {$\begin{bmatrix} 3\times3,\,128 \\ 3\times3,\,128 \end{bmatrix} \times 18$} \\
\hline
\end{tabular}
}
\caption{Details of the convolutional blocks of the auxiliary classifiers added to the ResNet backbone networks evaluated on the CIFAR-100 dataset. In the table, every cell shows the number of building blocks and the corresponding number of output channels.}
\label{tab:CIFAR_structures_resnet}
\vskip 0.1 in
\end{table}

\begin{table}[]
\centering
\scalebox{0.8}{
\begin{tabular}{|c|c|c|c|c|c|c|}
\hline
 & \multicolumn{3}{c|}{WRN-28-4} & \multicolumn{3}{c|}{WRN-28-10} \\
\hline
 & $C_1$ & $C_3$ & $C_2$ & $C_1$ & $C_3$ & $C_2$\\
\hline
 conv1 & {$3 \times 3,\,16$} & - & - & {$3 \times 3,\,16$} & - & -\\
\hline
 conv2\_x & {$\begin{bmatrix} 3\times3,\,64 \\ 3\times3,\,64 \end{bmatrix} \times 4$} & - & - & $\begin{bmatrix} 3\times3,\,160 \\ 3\times3,\,160 \end{bmatrix} \times 4$ & - & - \\
\hline
 conv3\_x & {$\begin{bmatrix} 3\times3,\,128 \\ 3\times3,\,128 \end{bmatrix} \times 4$} & {$\begin{bmatrix} 3\times3,\,128 \\ 3\times3,\,128 \end{bmatrix} \times 4$} & - & {$\begin{bmatrix} 3\times3,\,320 \\ 3\times3,\,320 \end{bmatrix} \times 4$} & {$\begin{bmatrix} 3\times3,\,320 \\ 3\times3,\,320 \end{bmatrix} \times 4$} & - \\
\hline
 conv4\_x & {$\begin{bmatrix} 3\times3,\,256 \\ 3\times3,\,256 \end{bmatrix} \times 4$} & {$\begin{bmatrix} 3\times3,\,256 \\ 3\times3,\,256 \end{bmatrix} \times 2$} & {$\begin{bmatrix} 3\times3,\,512 \\ 3\times3,\,512 \end{bmatrix} \times 4$} & {$\begin{bmatrix} 3\times3,\,640 \\ 3\times3,\,640 \end{bmatrix} \times 4$} & {$\begin{bmatrix} 3\times3,\,640 \\ 3\times3,\,640 \end{bmatrix} \times 2$} & {$\begin{bmatrix} 3\times3,\,1280 \\ 3\times3,\,1280 \end{bmatrix} \times 4$} \\
\hline
\end{tabular}
}
\caption{Details of the convolutional blocks of the auxiliary classifiers added to the WRN backbone networks evaluated on the CIFAR-100 dataset. In the table, every cell shows the number of building blocks and the corresponding number of output channels.}
\label{tab:CIFAR_structures_wrn}
\vskip 0.1 in
\end{table}

\begin{table}[]
\centering
\scalebox{0.8}{
\begin{tabular}{|c|c|c|c|c|c|c|}
\hline
 & \multicolumn{3}{c|}{DenseNet(d=40,k=12)} & \multicolumn{3}{c|}{DenseNet(d=100,k=12)} \\
\hline
 & $C_1$ & $C_3$ & $C_2$ & $C_1$ & $C_3$ & $C_2$\\
\hline
 conv1 & {$3 \times 3,\,24$} & - & - & {$3 \times 3,\,24$} & - & -\\
\hline
 conv2\_x & {$\begin{bmatrix} 3\times3,\,12 \end{bmatrix} \times 12$} & - & - & $\begin{bmatrix} 3\times3,\,12 \end{bmatrix} \times 32$ & - & - \\
\hline
 conv3\_x & {$\begin{bmatrix} 3\times3,\,12 \end{bmatrix} \times 12$} & {$\begin{bmatrix} 3\times3,\,12 \end{bmatrix} \times 12$} & - & {$\begin{bmatrix} 3\times3,\,12 \end{bmatrix} \times 32$} & {$\begin{bmatrix} 3\times3,\,12 \end{bmatrix} \times 16$} & - \\
\hline
 conv4\_x & {$\begin{bmatrix} 3\times3,\,12 \end{bmatrix} \times 12$} & {$\begin{bmatrix} 3\times3,\,12 \end{bmatrix} \times 6$} & {$\begin{bmatrix} 3\times3,\,36 \end{bmatrix} \times 12$} & {$\begin{bmatrix} 3\times3,\,12 \end{bmatrix} \times 32$} & {$\begin{bmatrix} 3\times3,\,12 \end{bmatrix} \times 16$} & {$\begin{bmatrix} 3\times3,\,36 \end{bmatrix} \times 32$} \\
\hline
\end{tabular}
}
\caption{Details of the convolutional blocks of the auxiliary classifiers added to the DenseNet backbone networks evaluated on the CIFAR-100 dataset. In the table, every cell shows the number of building blocks and the corresponding growth rate.}
\label{tab:CIFAR_structures_densenet}
\vskip 0.1 in
\end{table}

\onecolumn
\begin{table}[]
\centering
\scalebox{0.8}{
\begin{tabular}{|c|c|c|c|c|c|c|c|c|c|c|c|}
\hline
$C_1$ & 32 & 64 & 128 & 128 & 256 & 256 & (512,s2) & 512 $\times$ 5 & (1024,s2), 1024\\
\hline
$C_3$ & -  & -  & -   & -   & -   & -   & (512,s2) & 512 $\times$ 3 & (1024,s2), 1024\\
\hline
$C_2$ & -  & -  & -   & -   & -   & -   & -       & -             & (2048,s2), 2048\\
\hline
\end{tabular}
}
\caption{Details of the convolutional blocks of the auxiliary classifiers added to the MobileNet backbone network evaluated on the CIFAR-100 dataset. In the table, every cell shows the number of output channels, and $s2$ denotes the stride of the convolution operation in this layer is 2.}
\label{tab:CIFAR_structures_mobilenet}
\vskip 0.1 in
\end{table}

\begin{table}[]
\centering
\scalebox{0.8}{
\begin{tabular}{|c|c|c|c|c|c|c|c|}
\hline
 & \multicolumn{3}{c|}{Original} & \multicolumn{2}{c|}{Narrow} & \multicolumn{2}{c|}{Shallow} \\
\hline
 & $C_1$ & $C_3$ & $C_2$ & $C_3$ & $C_2$ & $C_3$ & $C_2$\\
\hline
 conv1 & {$3 \times 3,\,16$} & - & - & - & - & - & -\\
\hline
 conv2\_x & {$\begin{bmatrix} 3\times3,\,16 \\ 3\times3,\,16 \end{bmatrix} \times 5$} & - & - & - & - & - & -\\
\hline
 conv3\_x & {$\begin{bmatrix} 3\times3,\,32 \\ 3\times3,\,32 \end{bmatrix} \times 5$} & {$\begin{bmatrix} 3\times3,\,32 \\ 3\times3,\,32 \end{bmatrix} \times 5$} & - & {$\begin{bmatrix} 3\times3,\,16 \\ 3\times3,\,16 \end{bmatrix} \times 5$} & - & {$\begin{bmatrix} 3\times3,\,32 \\ 3\times3,\,32 \end{bmatrix} \times 2$} & - \\
\hline
 conv4\_x & {$\begin{bmatrix} 3\times3,\,64 \\ 3\times3,\,64 \end{bmatrix} \times 5$} & {$\begin{bmatrix} 3\times3,\,64 \\ 3\times3,\,64 \end{bmatrix} \times 3$} & {$\begin{bmatrix} 3\times3,\,128 \\ 3\times3,\,128 \end{bmatrix} \times 5$} & {$\begin{bmatrix} 3\times3,\,32 \\ 3\times3,\,32 \end{bmatrix} \times 3$} & {$\begin{bmatrix} 3\times3,\,64 \\ 3\times3,\,64 \end{bmatrix} \times 5$} & {$\begin{bmatrix} 3\times3,\,64 \\ 3\times3,\,64 \end{bmatrix} \times 1$} & {$\begin{bmatrix} 3\times3,\,128 \\ 3\times3,\,128 \end{bmatrix} \times 2$} \\
\hline
\end{tabular}
}
\caption{Details of the narrow and shallow auxiliary classifiers added to the ResNet-32 backbone network evaluated on the CIFAR-100 dataset. In the table, every cell shows the number of building blocks and the corresponding number of output channels.}
\label{tab:ep_study_FC}
\vskip 0.1 in
\end{table}

\begin{table}[]
\centering
\scalebox{0.6}{
\begin{tabular}{|c|c|c|c|c|c|c|c|}
\hline
 & \multicolumn{3}{c|}{ResNet-18} & \multicolumn{2}{c|}{ResNet-50} & \multicolumn{2}{c|}{ResNet-152} \\
 \hline
 & $C_2$ & $C_3$ & $C_4$ & $C_2$ & $C_3$ & $C_2$ & $C_3$ \\
\hline
conv3\_x & -      & -     & $\begin{bmatrix} 3\times3,\,128 \\ 3\times3,\,128 \end{bmatrix} \times 1$ & -     & -      & -      & - \\
\hline
conv4\_x & -      & $\begin{bmatrix} 3\times3,\,256 \\ 3\times3,\,256 \end{bmatrix} \times 1$ & $\begin{bmatrix} 3\times3,\,256 \\ 3\times3,\,256 \end{bmatrix} \times 1$ & -     & $\begin{bmatrix} 1\times1,\,256 \\ 3\times3,\,256 \\ 1\times1,\,1024 \end{bmatrix} \times 3$  & -      & $\begin{bmatrix} 1\times1,\,256 \\ 3\times3,\,256 \\ 1\times1,\,1024 \end{bmatrix} \times 18$ \\
\hline
conv5\_x & $\begin{bmatrix} 3\times3,\,1024 \\ 3\times3,\,1024 \end{bmatrix} \times 2$ & $\begin{bmatrix} 3\times3,\,512 \\ 3\times3,\,512 \end{bmatrix} \times 2$ & $\begin{bmatrix} 3\times3,\,512 \\ 3\times3,\,512 \end{bmatrix} \times 2$ & $\begin{bmatrix} 1\times1,\,1024 \\ 3\times3,\,1024 \\ 1\times1,\,4096 \end{bmatrix} \times 3$ & $\begin{bmatrix} 1\times1,\,512 \\ 3\times3,\,512 \\ 1\times1,\,2048 \end{bmatrix} \times 2$ & $\begin{bmatrix} 1\times1,\,1024 \\ 3\times3,\,1024 \\ 1\times1,\,4096 \end{bmatrix} \times 3$ & $\begin{bmatrix} 1\times1,\,512 \\ 3\times3,\,512 \\ 1\times1,\,2048 \end{bmatrix} \times 2$ \\
\hline
\end{tabular}
}
\caption{Details of the convolutional blocks of the auxiliary classifiers added to the ResNet backbone networks evaluated on the ImageNet classification dataset. In the table, every cell shows the corresponding number of convolutional blocks (including basic blocks for ResNet-18, and bottleneck blocks for ResNet-50 and ResNet-152) and their parameter sizes. For comparison with the backbone networks, please refer to the Table 1 of the ResNet paper~\cite{ref15}.}
\label{tab:structures_imagenet}
\vskip 0.1 in
\end{table}
~\\  

\end{document}